%% file: iclr2027_conference.tex
\documentclass{article}

\usepackage{iclr2027_conference,times}

\iclrfinalcopy

\usepackage{microtype}
\usepackage{graphicx}
\usepackage{booktabs}
\usepackage{multirow}
\usepackage{longtable}
\usepackage{hyperref}
\usepackage{url}
\usepackage{amsmath}
\usepackage{amssymb}
\usepackage{mathtools}
\usepackage{amsthm}
\usepackage{aliascnt}
\usepackage[capitalize,noabbrev]{cleveref}
\usepackage{tikz}
\usetikzlibrary{positioning,arrows.meta,calc,fit}
\usepackage{colortbl}
\usepackage{algorithm}
\usepackage{algorithmic}

\usepackage{enumitem}
\usepackage[skip=3pt,font=small]{caption}
\usepackage{titlesec}
\titlespacing*{\section}{0pt}{6pt plus 2pt minus 2pt}{4pt plus 1pt minus 1pt}
\titlespacing*{\subsection}{0pt}{5pt plus 2pt minus 2pt}{3pt plus 1pt minus 1pt}
\setlist{topsep=2pt,itemsep=1pt,parsep=0pt}

\definecolor{tgwin}{HTML}{D4EDDA}
\definecolor{tglose}{HTML}{FFF3CD}
\newcommand{\win}[1]{\cellcolor{tgwin}#1}
\newcommand{\lose}[1]{\cellcolor{tglose}#1}

\theoremstyle{plain}
\newtheorem{theorem}{Theorem}[section]
\newcommand{\sharedtheorem}[3]{%
  \newaliascnt{#1}{theorem}%
  \newtheorem{#1}[#1]{#2}%
  \aliascntresetthe{#1}%
  \crefname{#1}{#2}{#3}\Crefname{#1}{#2}{#3}}
\sharedtheorem{proposition}{Proposition}{Propositions}
\sharedtheorem{lemma}{Lemma}{Lemmas}
\sharedtheorem{corollary}{Corollary}{Corollaries}
\theoremstyle{definition}
\sharedtheorem{definition}{Definition}{Definitions}
\sharedtheorem{assumption}{Assumption}{Assumptions}
\theoremstyle{remark}
\sharedtheorem{remark}{Remark}{Remarks}

\DeclareMathOperator{\Var}{Var}
\newcommand{\CV}{\mathrm{CV}}

\title{Theory-Guided Instance-Dependent Early Stopping via Second-Order Loss Dynamics}

\author{%
  Ahad Bin Islam Shoeb\(^{1}\), Jamal Uddin Tanvin\(^{2}\), Kamrul Hasan\(^{3}\), Alve Reduan\(^{1}\), Imtiaz Ahmed\(^{4}\) \\
  \(^{1}\)University of Dhaka \quad \(^{2}\)Military Institute of Science and Technology \\
  \(^{3}\)Tennessee State University \quad \(^{4}\)Howard University \\
  \texttt{ahadbinislam-2019117810@cs.du.ac.bd, jamaluddintanvin@outlook.com,}\\
  \texttt{mhasan1@tnstate.edu, iam.reduan@gmail.com, imtiaz.ahmed@howard.edu}
}

\fancypagestyle{firstpage}{%
  \fancyhf{} 
  \lfoot{%
    \vspace*{-3em}
    \rule{0.3\textwidth}{0.4pt}\\
    \footnotesize The code is available at \url{https://github.com/ShoebAhad/theory-guided-ies}\\[4pt]
    \makebox[\textwidth][c]{\normalsize\thepage}
  }
  \renewcommand{\headrulewidth}{0pt}
  \renewcommand{\footrulewidth}{0pt}
}

\begin{document}
\definecolor{tgwin}{HTML}{D4EDDA}
\definecolor{tglose}{HTML}{FFF3CD}

\maketitle
\thispagestyle{firstpage}

\lhead{}

\begin{abstract}
Instance-dependent early stopping (IES) reduces training computation by skipping backpropagation for examples whose losses have stabilized, but the second-order loss-difference criterion used by IES lacks a theoretical justification. We develop Theory-Guided Instance-Dependent Early Stopping (TG-IES), which analyzes this criterion under explicit assumptions and derives a noise-calibrated mastery threshold and a conditional stability bound for instance exclusion. We show that second-order differencing has a favorable bias--variance trade-off over first- and third-order differences in a specific signal-to-noise regime, and that its signal contains both gradient-drift and curvature terms. TG-IES derives its stopping threshold and patience from training-run statistics rather than using the manually selected threshold of IES. Across CIFAR-10 and CIFAR-100, TG-IES saves 53--58\% of backpropagation instances in its main convolutional settings while keeping accuracy within 0.1--0.9 percentage points of full training; in the primary ResNet-18/CIFAR-10 setting, this corresponds to 35.7\% FLOP-normalized savings. The method also transfers to ViT-Tiny, where it achieves smaller savings and incurs a modest accuracy cost. These results provide a theoretical basis for second-order instance-dependent early stopping and show that theory-guided calibration can substantially reduce training computation under the assumptions and experimental settings studied. 
\end{abstract}

\section{Introduction}
\label{sec:intro}

The computational resources required to train deep learning models have increased substantially \citep{sevilla2022compute}. As a result, the energy use and carbon footprint of model training are also closely tied to the amount of computation required \citep{patterson2022carbon}. Algorithmic methods that reduce unnecessary forward and backward computation can lower this cost without changing the model architecture or hardware \citep{bartoldson2023computeefficient}.

Training modern neural networks repeatedly processes the same training instances, even though different instances are learned at different rates. Here, a \emph{training instance} means one individual training sample, such as an input--label pair in a classification dataset. Studies of training dynamics show that some instances become easy to learn and remain stable, while others continue to be difficult or important for learning \citep{toneva2019empirical,swayamdipta2020dataset,paul2021deep}. This difference in learning behavior suggests that not every instance needs the same amount of computation throughout training.

Several methods use this observation to reduce training cost. One group selects or prunes instances using static, one-shot, or early-estimated measures of importance, difficulty, or diversity \citep{paul2021deep,sorscher2022beyond,zheng2023coverage,maharana2024d2}. Another group makes the decision during training using signals such as loss, gradient information, or other measures of instance importance \citep{loshchilov2015online,katharopoulos2018not,jiang2019accelerating,mindermann2022prioritized,qin2024infobatch,bankes2024reducr,nagaraj2025coresets,jin2026orderdp}. These methods share the goal of spending more computation on instances that are still useful for learning instead of treating all instances equally at every step. Recent work also shows that an instance's \emph{loss trajectory}, meaning its sequence of loss values over training, can provide information about its learning behavior beyond a single loss value \citep{nagaraj2025coresets}.

Our work focuses on a more specific question: \emph{when does an individual instance no longer need a gradient update?} \citet{yuan2025instance} introduced \emph{Instance-dependent Early Stopping} (IES), which stops backpropagation for an instance when its loss becomes sufficiently stable. The IES criterion uses the second-order finite difference of the instance's loss:
\begin{equation} 
\label{eq:mastery} 
\big|\Delta^2 L_i(\theta^{(t)})\big| = \big|L_i^{(t)} - 2L_i^{(t-1)} + L_i^{(t-2)}\big| < \delta, 
\end{equation}
Here, $L_i^{(t)}$ is the loss of instance $i$ at epoch $t$. The first difference $L_i^{(t)}-L_i^{(t-1)}$ measures how much the loss moved and $\Delta^2L_i$ measures how much that move itself changed. So a small $|\Delta^2L_i|$ means the loss trajectory has become locally stable. In IES, an instance is \emph{mastered} when this value stays below $\delta$ for a fixed number of epochs, and its gradient computation is then skipped. ``Mastered'' is an operational name for a stable loss trajectory. It does not imply zero loss, a zero gradient, or perfect prediction.

IES shows that this simple loss-based signal can reduce training computation, but several questions remain open. The original work does not explain why the second-order difference is a useful mastery signal, how its noise compares with first- and third-order differences, or how it relates to the optimization dynamics. It also does not analyze how excluding instances changes the training trajectory. Finally, the mastery threshold $\delta$ and the patience value are manually chosen rather than derived from the training run.

\textbf{Research question.}
\emph{Can we explain why the second-order loss difference works as a mastery signal, derive its threshold from observable training dynamics, and characterize how excluding instances affects the training trajectory?}

\textbf{Our answer has three parts.} First, under a local near-convergence model (\cref{sec:theory}), we analyze the second-order loss difference and explain when it is a useful signal. We show that relative loss changes can remove the effect of an instance’s loss scale, that $N=2$ has a favorable bias–variance trade-off in a specific signal-to-noise regime and that the second-order difference contains both gradient-drift and curvature terms.

Second, we use these results to develop a practical training procedure called \textbf{Theory-Guided Instance-Dependent Early Stopping (TG-IES)} (\cref{sec:method}). TG-IES estimates the noise level during training and uses it to set the mastery threshold $\delta$. It also determines the patience value $K^*$, which specifies how many consecutive low-signal epochs an instance must have before it is excluded from training. We initialize $K^*$ using a conditional reference horizon $K_{\mathrm{safe}}$. Finally, we show that excluding these instances is equivalent to using an approximate gradient update. The difference between this update and the full-training update is exactly the combined gradient of the excluded instances. We then bound how much this difference can change the training trajectory.

Third, we evaluate the theory and the method across datasets, architectures, training schedules, and competing methods, measuring both computational savings and accuracy and directly testing the assumptions behind the stability analysis (\cref{sec:results,sec:generalization}). Proofs and further results are in the Appendix.

\begin{figure}[!htb]
\centering
\input{fig_overview_tikz}
\caption{TG-IES uses each instance’s loss trajectory to derive a second-order loss signal. Observed noise sets the mastery threshold $\delta$, and a conditional reference horizon $K_{\mathrm{safe}}$ initializes patience $K^*$. The mastery rule decides when an instance becomes passive and when it re-enters training. Skipping passive instances alters updates by the passive-set gradient $\omega_t$, governing drift and descent (dashed box). Experiments assess compute savings, accuracy, and the method’s assumptions.}
\label{fig:overview}
\end{figure}
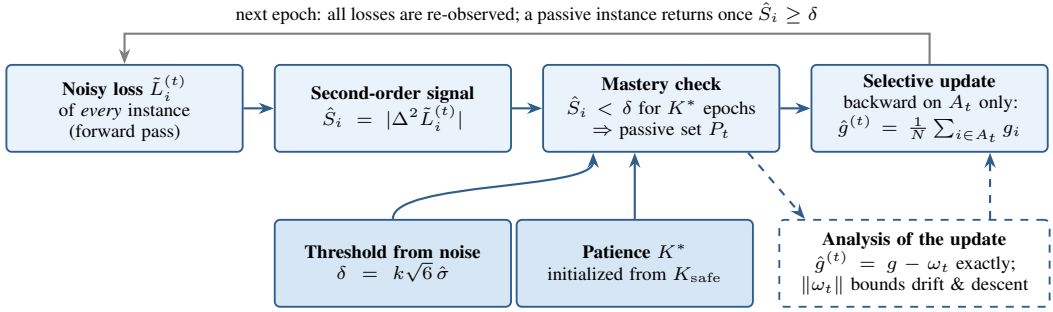

\textbf{Contributions.}
We provide a theoretical and empirical analysis of IES and develop Theory-Guided Instance-Dependent Early Stopping (TG-IES). Our contributions are:

\begin{itemize}[leftmargin=1.5em]

\item \textbf{Theory for the second-order loss signal.} We study when the second-order loss difference is useful. We analyze its ratio-level scale cancellation, the bias--variance trade-off, and the effects of gradient drift and curvature.

\item \textbf{Theory-guided calibration.} We use the noise observed during training to set the mastery threshold $\delta$. We also use a conditional reference horizon to initialize the patience value $K^*$.

\item \textbf{Conditional stability of the selective update.} We analyze how excluding instances affects training. We bound the resulting trajectory drift and loss descent using the combined gradient of the excluded instances.

\item \textbf{Empirical evaluation.} We evaluate TG-IES across different datasets, model architectures, optimizers, and training settings. We measure computation savings and accuracy, and evaluate the assumptions used in the stability analysis.

\end{itemize}

\section{Related Work}
\label{sec:related}

\textbf{Static data pruning.}
Static data pruning selects a fixed subset of training examples before or near the beginning of training. Existing methods use forgetting events \citep{toneva2019empirical}, early gradient-norm scores \citep{paul2021deep}, training-dynamics statistics \citep{swayamdipta2020dataset}, or scores obtained from a smaller proxy model \citep{coleman2020selection}. Other approaches consider dataset coverage, diversity, or scaling behavior when selecting the subset \citep{sorscher2022beyond,zheng2023coverage,maharana2024d2}. In contrast, IES and TG-IES make an instance-level decision at each epoch based on the current training trajectory.

\textbf{Dynamic and online data selection.}
Dynamic methods update the selection of training examples during training. Importance-sampling approaches focus on examples that contribute more to gradient estimation \citep{loshchilov2015online,katharopoulos2018not}, while Selective-Backprop \citep{jiang2019accelerating} uses a probabilistic rule to skip backpropagation for low-loss examples. Other methods use holdout models \citep{mindermann2022prioritized,bankes2024reducr}, epoch-level loss statistics \citep{qin2024infobatch}, trajectory-alignment scores \citep{qin2026alignprune,nagaraj2025coresets}, or moving averages of loss \citep{zhou2026bls}. \citet{jin2026orderdp} provide theoretical guarantees for a different dynamic-pruning mechanism. These methods mainly focus on deciding which examples to process or how examples are selected or ranked, whereas TG-IES calibrates its exclusion rule from the observed loss noise and analyzes the gradient error that exclusion introduces.

\textbf{Instance-Level Early Stopping.}
Curriculum learning controls when examples \emph{enter} training, typically progressing from easier to harder examples \citep{bengio2009curriculum,soviany2022curriculum}. Whereas IES and TG-IES determine when an example can \emph{leave}. Our direct predecessor is IES \citep{yuan2025instance}, which introduced instance-level early stopping using a second-order loss-difference signal. TG-IES extends IES with an analysis of this signal, noise-based threshold calibration, and a conditional stability analysis of the resulting selective update.

\section{Setup and Assumptions}
\label{sec:prelim}

We analyze training one instance at a time by studying the per-instance loss trajectory and then connecting it to the full training update. For each instance, we record how its loss changes from epoch to epoch. For analysis, training is idealized as gradient descent, $\theta^{(t+1)}=\theta^{(t)}-\eta_t\,g^{(t)}$ (constant $\eta_t\equiv\eta$ unless stated), where a superscript $(t)$ marks a value at $\theta^{(t)}$. \Cref{tab:notation} lists the symbols used throughout.

\begin{table}[t]
\caption{Notation used in the main text.}
\label{tab:notation}
\centering
\scriptsize
\setlength{\tabcolsep}{3pt}
\begin{tabular}{@{}l p{5.35cm} @{\hspace{8pt}} l p{5.0cm}@{}}
\toprule
Symbol & Meaning & Symbol & Meaning \\
\midrule
$L_i^{(t)}$ & loss of instance $i$ after epoch $t$ & $\theta^{(t)}$, $\eta_t$ & parameters after epoch $t$; step size \\
$\tilde L_i^{(t)}$ & recorded loss $L_i^{(t)}+\xi_{i,t}$ (noise scale $\sigma$) & $L$ & training loss $\frac1N\sum_iL_i$ \\
$\Delta^1L_i^{(t)}$ & $L_i^{(t)}-L_i^{(t-1)}$: how much the loss moved & $g_i$ & $\nabla_\theta L_i$: update direction from instance $i$ \\
$\Delta^2L_i^{(t)}$ & $\Delta^1L_i^{(t)}-\Delta^1L_i^{(t-1)}$: change of that move & $g$ & $\frac1N\sum_ig_i$: full-data gradient \\
$\Delta^NL_i^{(t)}$ & the difference applied $N$ times & $H_i$ & $\nabla^2_\theta L_i$: local curvature of $L_i$ \\
$\delta$, $K^*$ & mastery threshold on $|\Delta^2\tilde L_i|$; patience & $A_t$, $P_t$ & active / passive instances \\
\bottomrule
\end{tabular}
\end{table}
Thus, $\Delta^2L_i$ measures whether the per-epoch change in an instance's loss is still changing. \textbf{What we assume, and why.} Each result names the idealizations it uses; \cref{app:assumptions-extended} adds two technical ones (\cref{ass:cbounds,ass:thirdderiv}) and discusses failures, and proofs are in \cref{app:proofs}.

\begin{assumption}[Smoothness]
\label{ass:smooth}
$L$ is $\beta$-smooth: $\|g(\theta)-g(\theta')\|\le\beta\|\theta-\theta'\|$.
\emph{Purpose:} gradients cannot change arbitrarily fast, which turns per-step errors into multi-step bounds.
\end{assumption}

\begin{assumption}[Local geometric decay]
\label{ass:geometric}
Near a solution $\theta^*$, each excess loss shrinks by a constant fraction per epoch: $L_i^{(t)}-L^*=c_ir_i^t$, with scale $c_i>0$ and rate $r_i\in(0,1)$.
\emph{Purpose:} a simple late-training model in which finite differences are exact.
\end{assumption}

\begin{assumption}[Independent noise]
\label{ass:noise}
$\xi_{i,t}$ is conditionally mean-zero with variance $\sigma^2$, independent over epochs for fixed $i$.
\emph{Purpose:} noise does not persist, so its effect on $\Delta^N$ is exact (\cref{sec:noise-validation}).
\end{assumption}

\begin{assumption}[Sub-Gaussian noise]
\label{ass:subgaussian}
$\xi_{i,t}$ is conditionally $\sigma^2$-sub-Gaussian (needed only for tail bounds).
\end{assumption}

\begin{assumption}[Small rate dispersion]
\label{ass:dispersion}
$r_i\in[\bar r-\delta_r,\bar r+\delta_r]$ for all $i$, with $\delta_r$ small.
\emph{Purpose:} broadly similar decay rates, needed for one cross-instance comparison (\cref{prop:cv,app:counterexample}).
\end{assumption}

\section{Why the Second-Order Difference Works}
\label{sec:theory}

We answer three questions about the IES signal $\Delta^2L_i$, each as: question, result, meaning, limitation.

\subsection{Does the signal depend on the size of the loss?}
\label{sec:level1}

\textbf{Question.} Instances differ widely in loss size, but a useful signal should reflect \emph{how fast} an instance converges: two instances with different loss sizes and the same decay rate should look alike. Under \cref{ass:geometric}, $\Delta^N(L_i^{(t)}-L^*)=c_i\,r_i^{t-N}(r_i-1)^N$ (\cref{lem:geomdiff}), which gives the following.

\begin{theorem}[The loss scale cancels in the ratio]
\label{thm:ratio}
Under \cref{ass:geometric}, for every $N\ge1$ and every instance $i$,
\begin{equation}
\label{eq:ratio}
\frac{\Delta^N(L_i^{(t)}-L^*)}{\Delta^{N-1}(L_i^{(t)}-L^*)}=\frac{r_i-1}{r_i},
\end{equation}
where the denominator is nonzero for every finite $t$.
\end{theorem}

\textbf{Meaning.} The scale $c_i$ cancels, so instances with different loss sizes but equal rates give the same ratio: near convergence, differences reflect \emph{how} a loss changes, not only its size. With small rate dispersion, the spread of $|\Delta^2L_i|$ across instances also stays within $O(\delta_r)$ of that of $|\Delta^1L_i|$ (\cref{prop:cv}). \textbf{Limitation.} Only the \emph{ratio} is scale-free; $\Delta^2L_i$ itself still grows with $c_i$.

\subsection{Why the second order, rather than the first or third?}
\label{sec:level2}

\textbf{Question.} Each extra order of differencing removes more of the remaining trend in a loss trajectory, but it also amplifies measurement noise. Which order balances the two?

\begin{theorem}[Exact bias--variance decomposition]
\label{thm:bv}
Under \cref{ass:geometric,ass:noise},
\begin{equation}
\label{eq:bv}
\mathbb E\big[(\Delta^N\tilde L_i^{(t)})^2\big]
=\underbrace{c_i^2r_i^{2(t-N)}(r_i-1)^{2N}}_{\text{remaining trend}}
+\underbrace{\tbinom{2N}{N}\sigma^2}_{\text{noise}}.
\end{equation}
\end{theorem}

\textbf{Meaning.} The right side, $\mathrm{MSE}(N,t)$, is the mean squared size of the observed $N$-th difference. Raising $N$ shrinks its trend term (when $r_i>\tfrac12$) but grows its noise term combinatorially ($2$, $6$, $20$ for $N=1,2,3$), so the balance is explicit:

\begin{proposition}[When $N=2$ minimizes the error]
\label{prop:mse}
Under \cref{ass:geometric,ass:noise}, let $\alpha_i:=(1-r_i)^2/r_i^2$ and $r_i>\tfrac12$ (so $\alpha_i<1$). Then $N=2$ minimizes $\mathrm{MSE}(N,t)$ over $\{1,2,3\}$ if and only if $c_i^2r_i^{2t}$ lies in the non-empty signal-to-noise (SNR) band
\begin{equation}
\label{eq:band}
\frac{4\sigma^2}{\alpha_i(1-\alpha_i)}\;\le\;c_i^2r_i^{2t}\;\le\;\frac{14\sigma^2}{\alpha_i^2(1-\alpha_i)}.
\end{equation}
\end{proposition}

\textbf{Meaning.} 
The quantity $c_i^2r_i^{2t}$ represents the squared remaining signal, while $\sigma^2$ represents the observation noise. The result identifies an intermediate regime during training. Early on, the learning signal is strong enough to make higher-order differences useful. Very late in training, the signal becomes so weak that the extra noise of higher orders takes over, making the first-order difference ($N=1$) the better choice. The second-order difference strikes the ideal balance between these two extremes, which is exactly the phase targeted by the noise-calibrated TG-IES threshold. 
\textbf{Limitation.}
This does not show that $N=2$ is always optimal. It is optimal only when the signal lies in the band in \cref{eq:band}.

Pure noise gives $\Delta^2\xi_{i,t}=\xi_{i,t}-2\xi_{i,t-1}+\xi_{i,t-2}$; with independent terms, the coefficients $1,-2,1$ give variance $(1^2+2^2+1^2)\sigma^2=6\sigma^2$. The threshold can therefore be tied to the measured noise scale $\hat\sigma$ rather than fixed in loss units.

\begin{corollary}[Noise-calibrated threshold]
\label{cor:threshold}
Under \cref{ass:noise,ass:subgaussian}, $\Delta^2\xi_{i,t}$ has mean zero and variance $6\sigma^2$. With $\delta=k\sqrt6\,\sigma$, $\Pr(|\Delta^2\xi_{i,t}|>\delta)\approx2(1-\Phi(k))$ for approximately Gaussian noise ($\Phi$: standard normal CDF), and $\Pr(|\Delta^2\xi_{i,t}|>\delta)\le2\exp(-k^2/2)$ in general.
\end{corollary}

\textbf{Empirical check.} At the epoch when each instance's signal first reaches the calibrated threshold ($46{,}779$ instances), $N=2$ has the lowest empirical error, two and three orders of magnitude below $N=1$ and $N=3$, as predicted. A naive late-epoch average instead favors $N=1$, because already noise-dominated instances dominate it (\cref{app:extra-figs}).

\subsection{What does the second difference actually measure?}
\label{sec:level3}

\textbf{Question.} A second difference is often read as curvature, but the parameters also move between measurements. A step changes $L_i$ by about $-\eta\langle g_i,g\rangle$, so $D_i^{(t)}:=\langle g_i^{(t-2)},g^{(t-2)}\rangle-\langle g_i^{(t-1)},g^{(t-1)}\rangle$ measures the change in per-step progress on $i$. $O(\eta^k)$ marks a term scaling like $\eta^k$ for small $\eta$.

\begin{theorem}[What the second difference measures]
\label{thm:curvature}
Under \cref{ass:thirdderiv}, along gradient descent with step size $\eta$,
\begin{equation}
\label{eq:curvature}
\Delta^2L_i^{(t)}=\underbrace{\eta\,D_i^{(t)}}_{\substack{\text{change in}\\\text{progress, }O(\eta)}}
+\underbrace{\tfrac{\eta^2}{2}\big(g^{(t-1)\top}H_i(\theta^{(t-1)})g^{(t-1)}-g^{(t-2)\top}H_i(\theta^{(t-2)})g^{(t-2)}\big)}_{\text{change in curvature, }O(\eta^2)}
+\underbrace{R_{i,t}}_{\substack{\text{remainder,}\\O(\eta^3)}},
\end{equation}
with $|R_{i,t}|\le\frac{M_3}{6}\eta^3\big(\|g^{(t-1)}\|^3+\|g^{(t-2)}\|^3\big)$.
\end{theorem}

\textbf{Meaning.} 
This equation shows that $\Delta^2L_i$ should not be treated as a direct measure of curvature. Instead, it captures a mix of training dynamics: its main component is the change in the instance's per-step progress, while curvature only plays a secondary role at order $O(\eta^2)$. Because of this, it is better to interpret $\Delta^2L_i$ as a general stability or convergence signal. 
\textbf{Limitation.} This mathematical ordering assumes very small step sizes. At the actual step sizes used in our experiments, the progress (drift) term does not completely dominate and only correlates weakly with the observed signal (Pearson $r=0.23$--$0.37$, \cref{sec:lrmechanism}). Also, a small $\Delta^2L_i$  means the loss trajectory has flattened out and does not guarantee that the gradient itself is small (\cref{sec:bridge}).

\section{From the Signal to TG-IES}
\label{sec:method}
\label{sec:level4}

\subsection{The algorithm}
\label{sec:alg}
TG-IES retains the IES mastery check but calibrates its threshold and persistence from training statistics. At each epoch (\cref{alg:tgies}), it (1) forwards all instances and records $\tilde L_i^{(t)}$, (2) computes $\hat S_i=|\Delta^2\tilde L_i^{(t)}|$ once three observations are available, and (3) backpropagates only through the active set

\begin{equation}
\label{eq:activegrad}
\hat g^{(t)}=\frac1N\sum_{i\in A_t}g_i(\hat\theta^{(t)}),\qquad
\hat\theta^{(t+1)}=\hat\theta^{(t)}-\eta_t\hat g^{(t)}.
\end{equation}
An instance becomes passive after $K^*$ consecutive epochs with $\hat S_i<\delta$ and reactivates when $\hat S_i\ge\delta$. At epoch 30, $\delta=k\sqrt6\,\hat\sigma$ and $K^*$ are calibrated as described in \cref{app:practical}. Here, $k=\Phi^{-1}(1-p/2)$. The $1/N$ normalization preserves the full-training update scale.

\textbf{Patience.}
$K^*$ is based on the conditional reference horizon
$K_{\mathrm{safe}}=\lfloor N\tau/(\beta_{\mathrm{test}}\eta\varepsilon_0)\rfloor$
(\cref{cor:kstar}), clipped to $[1,K_{\max}]$. In all reported runs, $K_{\mathrm{safe}}>K_{\max}=15$, hence $K^*=15$.

\begin{algorithm}[!htb]
\caption{Theory-Guided Instance-Dependent Early Stopping (TG-IES).}
\label{alg:tgies}
\fontsize{8pt}{9pt}\selectfont
\begin{algorithmic}
\STATE {\bfseries Require:} data $D$, step sizes $\eta_t$, calibration epoch $t_c=30$, cap $K_{\max}$, epochs $T$
\STATE Init $\hat\theta^{(0)}$; $\mathrm{cnt}[i]=0$; all $i$ active
\FOR{$t=1$ {\bfseries to} $T$}
   \STATE Forward all $i\in D$; record $\tilde L_i^{(t)}$
   \IF{$t<3$}
      \STATE $\hat S_i\leftarrow\infty$
   \ELSE
      \STATE $\hat S_i\leftarrow|\tilde L_i^{(t)}-2\tilde L_i^{(t-1)}+\tilde L_i^{(t-2)}|$
   \ENDIF
   \IF{$t=t_c$}
      \STATE Calibrate $\hat\sigma,\varepsilon_0,\delta$ and $K^*$
   \ENDIF
   \IF{$t\ge t_c$}
      \FOR{each $i\in D$}
         \IF{$\hat S_i<\delta$}
            \STATE $\mathrm{cnt}[i]\mathrel{+}=1$; mark passive if $\mathrm{cnt}[i]\ge K^*$
         \ELSE
            \STATE $\mathrm{cnt}[i]\leftarrow0$; mark active
         \ENDIF
      \ENDFOR
   \ENDIF
   \STATE $A_t\leftarrow\{i:\text{active}\}$; \textbf{break} if $A_t=\emptyset$
   \STATE $\hat g^{(t)}\leftarrow\frac1N\sum_{i\in A_t}\nabla_\theta L_i(\hat\theta^{(t)})$;
   $\hat\theta^{(t+1)}\leftarrow\hat\theta^{(t)}-\eta_t\hat g^{(t)}$
\ENDFOR
\STATE \textbf{return} $\hat\theta^{(T)}$
\end{algorithmic}
\end{algorithm}

\subsection{What happens to optimization when instances are excluded?}
\label{sec:actual}

\textbf{Question.} Once instances stop receiving updates, how far can the model move from full training, and does it still descend? We follow the chain \emph{excluded instances $\to$ omitted gradient $\to$ deviation}.

\textbf{Step 1: the omitted gradient.}\label{sec:inexact} Since $A_t$ and $P_t=D\setminus A_t$ partition the data, the TG-IES gradient is
\begin{equation}
\label{eq:inexact}
\underbrace{\hat g^{(t)}}_{\text{TG-IES update}}=\underbrace{g(\hat\theta^{(t)})}_{\text{full-data gradient}}-\underbrace{\omega_t(\hat\theta^{(t)})}_{\text{omitted gradient}},
\qquad\text{where}\quad
\omega_t(\theta):=\frac1N\sum_{i\in P_t}g_i(\theta).
\end{equation}
\textbf{Meaning.} At the TG-IES parameters, the selective update equals the full-data gradient minus the aggregate gradient of the excluded instances. We call this term the \emph{omitted-gradient residual} $\omega_t$. This identity is exact and does not require the excluded instances to have individually small gradients. Their gradients may be large but cancel, or individually small but add together. The stability question is therefore about their \emph{aggregate} gradient. Thus, TG-IES can be analyzed using the established framework of inexact gradient methods \citep{So2013,Sun2017,HallakLevy2024}. Note that \cref{eq:inexact} simply describes the mathematical effect of excluding instances; the algorithm itself never actually computes $\omega_t$.

\textbf{Step 2: the deviation.} Let $\theta^{(t)}$ (full data) and $\hat\theta^{(t)}$ (TG-IES) start from the same parameters at epoch $t_0$; these are the trajectories that the two procedures actually produce.

\begin{theorem}[Actual-trajectory divergence]
\label{thm:actual}
Under \cref{ass:smooth}, suppose $\|\omega_t(\hat\theta^{(t)})\|\le\rho_t$ for $t=t_0,\dots,t_0+H-1$. Then
\begin{equation}
\label{eq:actual}
\big\|\hat\theta^{(t_0+H)}-\theta^{(t_0+H)}\big\|
\le\sum_{s=0}^{H-1}\eta_{t_0+s}\,\rho_{t_0+s}\prod_{j=s+1}^{H-1}\big(1+\beta\eta_{t_0+j}\big).
\end{equation}
For a constant step size $\eta$ and $\rho_t\le\rho$, the right side is at most $\frac{\rho}{\beta}\big[(1+\beta\eta)^H-1\big]$.
\end{theorem}

\textbf{Meaning.} The sum accumulates the omitted-gradient errors over the $H$ selective steps, and the product accounts for how an earlier error is amplified as the two trajectories evolve (proof in \cref{app:multistep}). The key stability quantity is thus the \emph{aggregate} omitted gradient: large individual passive gradients can partly cancel, and many small ones can add up. With one passive instance, the bound reduces as $\beta\to0$ to our earlier single-removal certificate $H\eta\varepsilon/N$ (\cref{thm:safe}), which ignores how a gap feeds back through the gradient. 
\textbf{Limitation.} The factor $(1+\beta\eta)^H$ grows with $H$, so the bound is informative over short horizons (\cref{rem:contraction-direction}), and it assumes a bound $\rho_t$ that we do not measure.

\textbf{Step 3: descent.}\label{sec:descent} Staying close to full training is more than optimization needs; it suffices that each step still decreases the full loss.

\begin{theorem}[Descent under relative omission error]
\label{thm:descent}
Under \cref{ass:smooth}, if $\|\omega_t(\hat\theta^{(t)})\|\le\nu\,\|\hat g^{(t)}\|$ for some $\nu\in[0,1)$ and $\eta_t<2(1-\nu)/\beta$, then
\begin{equation}
\label{eq:descent}
L(\hat\theta^{(t+1)})\le L(\hat\theta^{(t)})-\eta_t\Big(1-\nu-\frac{\beta\eta_t}{2}\Big)\|\hat g^{(t)}\|^2.
\end{equation}
\end{theorem}

\textbf{Meaning.} If the gradient that TG-IES leaves out is at most a fraction $\nu<1$ of the gradient it uses, every step decreases the full training loss, under a step-size limit that tightens as $\nu$ grows; $\nu=0$ gives the classical $\eta_t<2/\beta$ (\cref{app:multistep}). Whether $\nu<1$ holds in our runs is not measured.

\subsection{Connecting mastery to the residual}
\label{sec:bridge}

\textbf{Question.} Both theorems need a small omitted gradient, but the algorithm observes only $\hat S_i$. Does a small signal imply a small instance gradient? \textbf{Important distinction.} The next result is a \emph{sufficient theoretical route} under two further idealizations. It does not show that passive instances in deep networks satisfy them; \cref{sec:probe} tests this empirically.

\begin{assumption}[Per-instance smoothness and a shared solution]
\label{ass:smoothi}
Each $L_i$ is $\Lambda_i$-smooth on the neighborhood $U$ of \cref{ass:geometric}, and $L_i(\theta^*)=L^*$ at the \emph{same} $\theta^*$ for every $i$.
\end{assumption}

\begin{theorem}[Mastery bounds the instance gradient]
\label{thm:bridge}
Under \cref{ass:geometric,ass:smoothi} and the regularity condition of \cref{lem:descent} for $f=L_i$,
\begin{equation}
\label{eq:bridge}
\|g_i^{(t)}\|\le\frac{r_i}{1-r_i}\sqrt{2\Lambda_i\,\Delta^2L_i^{(t)}},
\quad\text{so}\quad
\Delta^2L_i^{(t)}<\delta\;\Rightarrow\;\|g_i^{(t)}\|<\varepsilon_i(\delta):=\frac{r_i}{1-r_i}\sqrt{2\Lambda_i\delta}.
\end{equation}
\end{theorem}

\textbf{Meaning.} If every passive instance satisfies these assumptions along the TG-IES trajectory and $\varepsilon_i(\delta)\le\bar\varepsilon(\delta)$, then, by the triangle inequality,
\begin{equation}
\label{eq:aggregate}
\|\omega_t\|\;\le\;\frac1N\sum_{i\in P_t}\|g_i\|\;<\;\frac1N\sum_{i\in P_t}\varepsilon_i(\delta)\;\le\;\frac{|P_t|}{N}\,\bar\varepsilon(\delta),
\end{equation}
so $\rho_t=(|P_t|/N)\,\bar\varepsilon(\delta)$ links the threshold to \cref{thm:actual} (under noise, use $\delta+k\sqrt6\,\sigma$, except with probability $2e^{-k^2/2}$; \cref{cor:bridgeprob}). \textbf{Limitation.} $r_i$ and $\Lambda_i$ are not estimated, a shared solution is not established for deep networks, and $\varepsilon_i(\delta)$ is uninformative as $r_i\to1$. Passing instances do have $30$--$33\times$ smaller mean gradient norms (\cref{app:extra-figs}), but the per-instance bound often fails (\cref{sec:probe}), so we state stability through the aggregate residual.

\section{Experimental Setup}
\label{sec:setup}

\textbf{Methods.} Four methods share identical code and optimizer settings: \textbf{Baseline} (no removal); \textbf{IES (shipped)}, the released code (smoothed signal, train/eval mismatch, $1/|A_t|$ scaling); \textbf{IES (literal)}, our faithful implementation of the stated IES algorithm \citep{yuan2025instance} (unsmoothed signal, consistent passes, $1/|A_t|$ scaling); and \textbf{TG-IES}. We also run InfoBatch \citep{qin2024infobatch}, BLS-InfoBatch \citep{zhou2026bls}, AlignPrune \citep{qin2026alignprune}, and Selective Backprop \citep{jiang2019accelerating} (\cref{app:newbaselines}).

\textbf{Data, models, and training.} We use CIFAR-10/100 \citep{krizhevsky2009learning} with ResNet-18 \citep{he2016deep} (primary), VGG-16 \citep{simonyan2015very}, DenseNet-121 \citep{huang2017densely}, and a CIFAR-native ViT-Tiny \citep{dosovitskiy2021image} trained from scratch (\cref{app:vit}). CNNs use SGD (momentum $0.9$, weight decay $5\times10^{-4}$, batch size $64$, $200$ epochs) with exponential decay ($\eta_0=0.1$, $\gamma=0.96$ per epoch) or a fixed rate ($\eta=10^{-3}$), plus Adam \citep{kingma2015adam} and single-seed linear, step, and cosine schedules. The ViT uses AdamW with warmup and cosine decay.

\textbf{One calibration policy.} TG-IES derives $\delta$ and $K^*$ from each run's own $\hat\sigma$ and $\varepsilon_0$, with constants fixed for all eleven settings ($p=0.3$, $\tau=0.02$, $\beta_{\mathrm{test}}=1$, the 20th percentile for $\varepsilon_0$, $K_{\max}=15$, one recalibration at epoch 30). Unlike the dataset-specific multiplier in the released IES code, none is retuned, so we call this \emph{noise-adaptive threshold calibration}, not a parameter-free method (\cref{sec:sensitivity}).

\textbf{Metrics.} Accuracy is the best test accuracy over $200$ epochs \citep{yuan2025instance}. Backprop savings are $100\times(1-\text{backpropagated instances}/(N\times\text{epochs}))$. Main settings use seeds $\{0,1,2\}$ and report mean $\pm$ std, with paired statistics in \cref{app:stats}. All runs use one RTX 5090 GPU.

\section{Results}
\label{sec:results}

\begin{table}[!htb]
\centering
\scriptsize
\begin{minipage}[t]{0.45\linewidth}
\centering
\captionof{table}{Primary setting (ResNet-18, CIFAR-10, SGD, exponential decay), 3 seeds.}
\label{tab:primary}
\setlength{\tabcolsep}{3pt}
\begin{tabular}{lcc}
\toprule
Method & Acc.\ (\%) & Saved (\%) \\
\midrule
Baseline & $95.03 \pm 0.13$ & $0.03 \pm 0.00$ \\
IES (shipped) & $94.71 \pm 0.09$ & $45.66 \pm 0.38$ \\
IES (literal) & $94.85 \pm 0.17$ & $49.02 \pm 0.44$ \\
InfoBatch & $94.32\pm0.12$ & $38.91\pm0.01$ \\
BLS-InfoBatch & $94.68\pm0.18$ & $35.84\pm0.09$ \\
AlignPrune & $94.48\pm0.05$ & $20.83\pm0.07$ \\
\win{TG-IES (ours)} & $94.85 \pm 0.28$ & \win{$\mathbf{53.63 \pm 0.30}$} \\
\bottomrule
\end{tabular}
\end{minipage}\hfill
\begin{minipage}[t]{0.53\linewidth}
\centering
\captionof{table}{Component ablation (primary setting, seed 0). Starred: calibrated; others keep IES (literal)'s values. Restricted TG-IES variants, not competitors.}
\label{tab:ablation}
\setlength{\tabcolsep}{3pt}
\begin{tabular}{lcccc}
\toprule
Calibrated & Acc.\ (\%) & Saved (\%) & $\delta$ & $K^*$ \\
\midrule
Neither & 94.92 & 48.59 & 0.001 & 1 \\
$\delta$ only & 94.68 & \textbf{66.62} & 0.0077* & 1 \\
$K^*$ only & 95.04 & 3.41 & 0.001 & 15* \\
Both (TG-IES) & 94.61 & 53.48 & 0.0160* & 15* \\
\bottomrule
\end{tabular}
\end{minipage}
\end{table}

\textbf{Primary setting.}
TG-IES achieves the highest instance-level backward savings among the compared methods while matching IES (literal)'s mean accuracy. The paired seed results have the same savings advantage in all three seeds (\cref{app:stats}). The cumulative trajectories in \cref{fig:backprop} show that the advantage emerges mainly in the latter half of training.

\begin{figure}[!htb]
\centering
\begin{minipage}[t]{0.48\linewidth}
\centering
\includegraphics[width=0.85\linewidth]{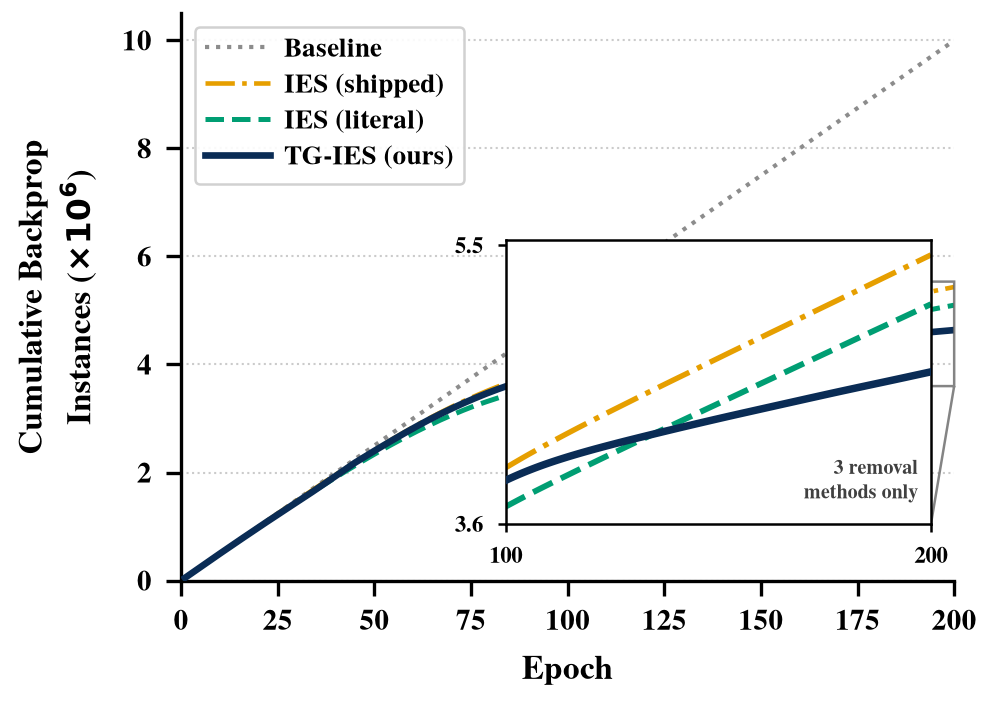}
\caption{Cumulative backpropagation instances for ResNet-18/CIFAR-10. TG-IES uses fewer backward updates than both IES variants late in training while maintaining similar test accuracy. Inset: instance-dependent methods over epochs 100--200.}
\label{fig:backprop}
\end{minipage}\hfill
\begin{minipage}[t]{0.48\linewidth}
\centering
\includegraphics[width=0.85\linewidth]{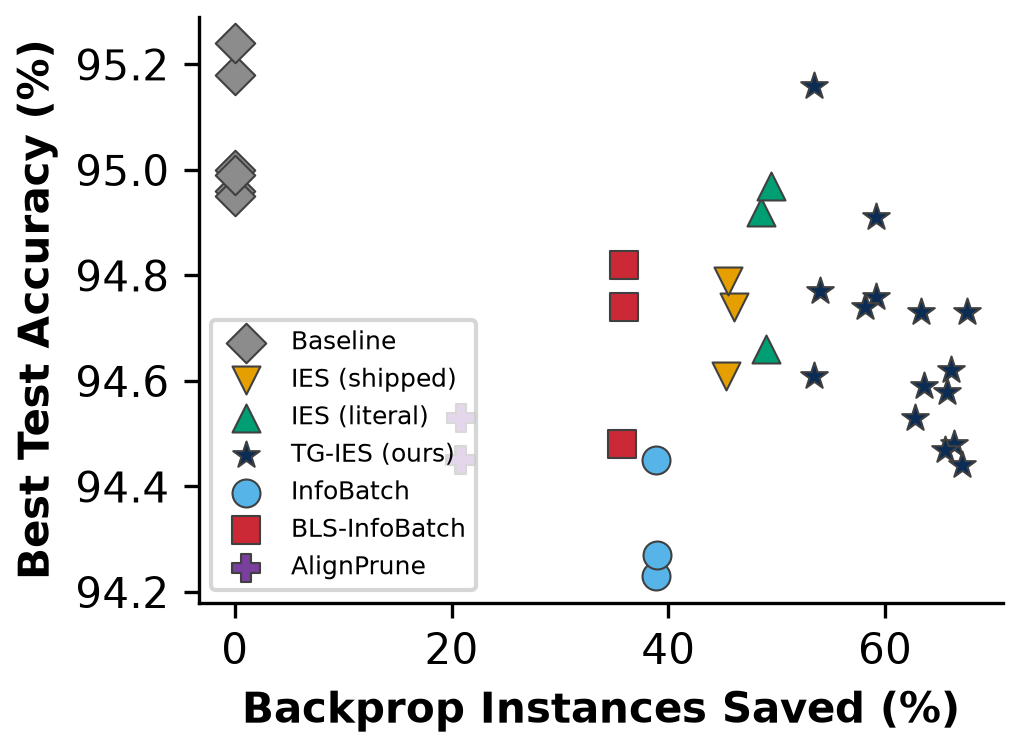}
\caption{Best test accuracy vs.\ backpropagation instances saved across the primary and additional baseline comparisons. TG-IES achieves substantial savings while remaining closer to full-training accuracy than more aggressive pruning methods.}
\label{fig:newbaselines}
\end{minipage}
\end{figure}

\textbf{What the savings mean in compute.}\label{sec:flops} 
Every IES-family method still forwards every instance, and a backward pass costs about twice a forward pass (measured: $1.995$--$1.998\times$), so FLOP savings are $\tfrac23$ of backprop savings: $\mathbf{35.7\%}$ for TG-IES, versus $30.4\%$ and $32.7\%$ for IES (shipped) and IES (literal). End-to-end (\cref{app:efficiency}), TG-IES is $25.5\%$ faster than full training and uses an estimated $39\%$ less energy but has the highest peak GPU memory ($1454$ vs.\ $1162$ MiB), from its one-time recalibration.

\begin{table}[!htb]
\caption{All three-seed settings, mean $\pm$ std ($^\dagger$primary). IES (literal), the higher-saving IES variant in every setting, is shown; all runs are in \cref{tab:fullresults}.}
\label{tab:overview}
\centering
\scriptsize
\setlength{\tabcolsep}{4pt}
\begin{tabular}{lccccc}
\toprule
& Baseline & \multicolumn{2}{c}{IES (literal)} & \multicolumn{2}{c}{TG-IES (ours)} \\
\cmidrule(lr){2-2}\cmidrule(lr){3-4}\cmidrule(lr){5-6}
Setting & Acc.\ (\%) & Acc.\ (\%) & Saved (\%) & Acc.\ (\%) & Saved (\%) \\
\midrule
\multicolumn{6}{l}{\emph{SGD with exponential learning-rate decay}} \\
ResNet-18 / CIFAR-10$^\dagger$ & $95.03\pm0.13$ & $94.85\pm0.17$ & $49.02\pm0.44$ & $94.85\pm0.28$ & \win{$53.63\pm0.30$} \\
VGG-16 / CIFAR-10 & $93.30\pm0.12$ & $93.15\pm0.17$ & $41.02\pm1.12$ & $93.11\pm0.14$ & \win{$54.43\pm0.12$} \\
DenseNet-121 / CIFAR-10 & $95.36\pm0.14$ & $95.24\pm0.02$ & $51.60\pm0.89$ & $95.28\pm0.09$ & \win{$54.53\pm0.66$} \\
ResNet-18 / CIFAR-100 & $76.88\pm0.41$ & $76.78\pm0.40$ & $18.49\pm0.43$ & $76.01\pm0.08$ & \win{$57.76\pm0.04$} \\
\midrule
\multicolumn{6}{l}{\emph{No learning-rate decay (ResNet-18 / CIFAR-10)}} \\
SGD, fixed rate & $92.21\pm0.08$ & $91.77\pm0.11$ & $49.42\pm0.17$ & $91.91\pm0.18$ & $29.27\pm0.30$ \\
Adam & $92.91\pm0.06$ & $92.42\pm0.10$ & $60.09\pm0.29$ & $92.44\pm0.10$ & $30.18\pm0.62$ \\
\midrule
\multicolumn{6}{l}{\emph{Vision Transformer (AdamW, warmup and cosine decay)}} \\
ViT-Tiny / CIFAR-10 & $84.07\pm0.35$ & $83.86\pm0.48$ & $16.68\pm0.32$ & $83.71\pm0.47$ & \win{$25.52\pm0.26$} \\
ViT-Tiny / CIFAR-100 & $59.25\pm0.24$ & $58.77\pm0.24$ & $12.51\pm0.21$ & $58.56\pm0.47$ & \win{$26.28\pm0.14$} \\
\bottomrule
\end{tabular}
\end{table}

\Cref{fig:newbaselines} summarizes the accuracy--compute trade-off across the primary and additional baseline comparisons. It provides a visual summary of how TG-IES compares with the IES variants and other dynamic
pruning methods.

\textbf{Where the gains hold, and where they do not.}\label{sec:settings} \Cref{tab:overview} shows two regimes. \emph{With decaying learning rates, TG-IES saves more than both IES variants:} $53.6$--$57.8\%$ in all four exponential-decay CNN settings, within $0.2$ accuracy points of full training on CIFAR-10. The largest gain is on CIFAR-100 ($57.76\%$ vs.\ $18.49\%$ for IES (literal)), which also has the largest CNN accuracy cost ($0.87$ points).

\emph{Without decay, IES saves more:} TG-IES saves $29.27\%$ vs.\ $49.42\%$ with a fixed rate and $30.18\%$ vs.\ $60.09\%$ with Adam, at accuracy within $0.15$ points of IES (literal). There the calibrated threshold is much smaller ($\delta=0.00247$ and $0.000257$ at epoch 31, vs.\ $0.01597$ with decay), so fewer instances complete $K^*$ passing checks and the final excluded fraction falls from $87.4\%$ to $42.0\%$ and $45.6\%$ (\cref{sec:lrmechanism} examines why; \cref{app:zeroshot} adds three schedules). The advantage is thus schedule-specific.

\textbf{Beyond CNNs.} On ViT-Tiny, the same constants save $25.5$--$26.3\%$, $1.5$--$2.1\times$ IES (literal), with a larger accuracy cost (up to $0.69$ points) and the widest seed spread of any method (\cref{app:vit}): the mechanism transfers, the magnitude does not.

\textbf{Matched compute.} At a matched budget (\cref{app:equalcompute}), InfoBatch and BLS-InfoBatch, which skip both passes, lead TG-IES by $2.5$--$4.3$ accuracy points at $20$--$30\%$ of full compute; the gap narrows to $0.1$--$0.8$ points by $40\%$ and reverses on CIFAR-100. Selective Backprop skips $66.72\%$ of backpropagation but loses $1.72$ accuracy points. TG-IES is one point on a compute--accuracy trade-off, and below about $30\%$ of full compute, InfoBatch-family methods reach higher accuracy in these runs.

\section{What Drives the Savings, and How Stable Is Exclusion?}
\label{sec:generalization}

\textbf{Component ablation.}\label{sec:ablation} \Cref{tab:ablation} calibrates $\delta$ and $K^*$ separately, keeping the other at its IES (literal) value. Calibrating $\delta$ alone raises savings from $48.6\%$ to $66.6\%$, exceeding full TG-IES ($53.5\%$) at $0.24$ points lower accuracy, whereas $K^*=15$ with IES's strict $\delta=0.001$ saves only $3.4\%$ because noise resets the counter before 15 consecutive passes. Thus, $\delta$ drives the compute benefit, while $K^*$ controls exclusion conservativeness; both must be calibrated to the same noise.

\textbf{Patience sweep and trajectory divergence.}\label{sec:kstartraj} Across $K^*\in{1,2,4,8,15,30}$ (primary setting, 3 seeds), savings fall from $66.98\%$ to $44.13\%$ as accuracy rises from $94.55\%$ to $94.87\%$ (baseline $95.08\%$; \cref{app:kstartraj}). The parameter distance $|\hat\theta^{(t)}-\theta^{(t)}_{\mathrm{baseline}}|$ to a shared-seed full-training run decreases for every $K^*$, from about $54$ at epoch 20 to $29$--$37$ at epoch 200, then plateaus. Larger $K^*$ consistently gives smaller distance, as \cref{thm:actual} suggests, because longer patience keeps the passive set smaller. This is consistent with a bounded aggregate residual but does not establish it; the norm-normalized distance still rises before plateauing.

\textbf{Does the per-instance gradient condition hold?}\label{sec:probe} Every 20 epochs, a non-intrusive probe computes the true gradient of 300 random passive instances (primary setting, 3 seeds). Across 27 probes, $47.4\%$ exceed $\varepsilon_0$ on average, with no downward trend; the median remains near $\varepsilon_0$, while a heavy tail makes the mean $2$--$10\times$ higher (\cref{app:reactprobe}). Thus, the per-instance floor underlying $K_{\mathrm{safe}}$ is too strong empirically, motivating the aggregate residual $\omega_t$ for stability. Since the probe does not measure $\omega_t$, the condition of \cref{thm:actual} remains an assumption.

\section{Limitations}
\label{sec:limitations}

\textbf{Conditional stability.} \Cref{thm:actual,thm:descent} assume a bound on $\omega_t$ along the actual trajectory that the implementation neither estimates nor controls, so they identify what governs stability without certifying TG-IES globally. With about half of the sampled passive instances above the per-instance floor, $K_{\mathrm{safe}}$ only initializes $K^*$; monitoring $\omega_t$ online is future work. \Cref{thm:bridge} also needs per-instance smoothness and a shared interpolating solution, not established for deep networks. \textbf{Scope.} The theory assumes full-batch gradient descent, unlike our SGD, Adam, and AdamW runs; without learning-rate decay, IES saves more; the evidence covers CIFAR-scale data and one small ViT; and most of the savings advantage comes from $\delta$, not $K^*$.

\section{Conclusion}
\label{sec:conclusion}
We study instance-dependent early stopping using a second-order per-instance loss signal. Under explicit assumptions, its normalized differences remove loss-scale effects; $N=2$ minimizes the MSE in a defined SNR regime; and the signal captures per-instance progress and curvature. We calibrate the mastery threshold from the observed noise scale and model selective updates as full-data gradient descent plus an omitted-gradient residual, yielding conditional trajectory and descent bounds. Empirically, TG-IES cuts instance-level backward passes in the main setting, with gains over IES depending on the learning-rate schedule. TG-IES is thus a noise-adaptive, theory-grounded IES method with stability guarantees conditional on the omitted-gradient residual.

\section*{Impact Statement}
This paper advances compute-efficient training; we identify no societal consequences requiring specific highlighting.

\section*{AI Use Statement}
In this paper we used generative AI tools (ChatGPT, Quillbot, and claude) to assist with our research and writing. The key theoretical frameworks, mathematical claims, algorithmic designs, and codebases were developed solely by the authors. We used Claude exclusively for code generation help. Additionally, we sought assistance from ChatGPT and Quillbot to enhance language, rearrange for readability, revise author-drafted parts, and format related literature. We have rigorously reviewed all AI-assisted work. All mathematical formulations and proofs were independently derived and verified step-by-step by the authors to ensure correctness and theoretical soundness. Likewise, any AI-assisted code was manually reviewed, integrated, and empirically tested. Generative AI was not used to run experiments or conduct data analysis, and all reported numbers come directly from our own experiments. We take responsibility for the final content of this work, including text, claims, or artifacts produced with the aid of generative AI.

\bibliography{iclr2027_conference}
\bibliographystyle{iclr2027_conference}

\newpage
\appendix
\onecolumn

\section{Extended Related Work}
\label{app:related}

\textbf{Why training-time compute matters.} Training compute for state-of-the-art models has grown roughly $10\times$ every year since 2012 and continues to grow \citep{sevilla2022compute}, with a directly proportional energy and carbon cost \citep{patterson2022carbon}. \citet{bartoldson2023computeefficient}'s survey of algorithmic (rather than hardware) efficiency gains frames exactly the category of method this paper contributes to: a technique reducing the number of forward/backward passes a fixed training recipe performs, orthogonal to and stackable with hardware- or precision-level efficiency gains. \Cref{sec:flops}'s FLOP-normalized accounting of TG-IES's savings is this paper's own direct, measured contribution to that question.

\textbf{Static data pruning and difficulty scoring.} \citet{toneva2019empirical} track how often an example is learned then forgotten across epochs and show unforgettable examples can be removed with little accuracy loss. \citet{paul2021deep} score examples early in training with the gradient-norm-based EL2N score and prune a static subset before the rest of training. \citet{swayamdipta2020dataset} map examples along confidence and variability axes computed from training dynamics (``dataset cartography''). \citet{coleman2020selection} select a subset using scores from a small, cheap proxy model rather than the target model itself. \citet{sorscher2022beyond} show, empirically and via a statistical-mechanics analysis of a toy model, that a good pruning metric can convert the usual power-law scaling of test error with dataset size into exponential scaling; \citet{zheng2023coverage} (CCS) show that pruning to the highest-difficulty examples alone hurts distributional coverage at high pruning rates and instead balance importance against coverage explicitly; \citet{maharana2024d2} (D2 Pruning) use a message-passing formulation over a similarity graph to balance the same trade-off. All of these commit to a fixed subset before or shortly after training starts; IES and TG-IES instead make a per-instance, per-epoch decision using the current trajectory.

\textbf{Dynamic and online example selection.} \citet{loshchilov2015online} rank minibatches online by loss to bias sampling toward harder examples. \citet{katharopoulos2018not} use importance sampling to reduce gradient variance, building on \citet{needell2014stochastic}. \citet{jiang2019accelerating} (Selective-Backprop) skip backpropagation for examples with low relative loss via a probabilistic accept/reject rule, close in spirit to IES but formalized on a different argument (a probabilistic threshold on the loss distribution, versus a finite-difference criterion on the loss trajectory), with no theoretical treatment on either side. \citet{mindermann2022prioritized} (RHO-LOSS) select points simultaneously learnable, worth learning, and not yet learnt via a holdout-model criterion; \citet{bankes2024reducr} (REDUCR) extend this with class-priority reweighting for robustness, orthogonal to the calibration question here. \citet{qin2024infobatch} (InfoBatch) prune examples below the epoch-mean loss with probability $r$, rescaling kept low-score instances' gradients by $1/(1-r)$, annealing back to the full dataset for the final epochs; unlike IES/TG-IES it skips both forward and backward on a pruned example. \citet{qin2026alignprune} (AlignPrune) replace loss-magnitude ranking with a Dynamic Alignment Score (correlation between an example's loss trajectory and a reference trajectory) for label-noise robustness, a different objective than mastery-based removal; \citet{nagaraj2025coresets} use a related trajectory-correlation signal for one-shot coreset selection against a held-out reference, with convergence guarantees for the resulting subset -- the closest published criterion to IES/TG-IES's own signal, differing in using a held-out reference rather than a self-referential finite difference, and selecting once rather than every epoch. \citet{zhou2026bls} (BLS) use an EMA of loss as a cheap, signal-processing alternative to TG-IES's noise-calibrated $\delta$ for the same denoising problem. Concurrently, \citet{jin2026orderdp} (OrderDP) give a formal guarantee for dynamic pruning via an unbiased ranking surrogate with a lossless-convergence proof, a different mechanism than the conditional, residual-based stability analysis here -- two different answers to the same open question rather than competitors on the same construction. \citet{zhu2026dppo} address an analogous unbiased-pruning problem for group-based RL policy optimization (GRPO rollout pruning for LLM reasoning), noted for completeness; its setting (sequence-level RL) is different enough that we do not attempt a direct comparison. The same question -- what to spend a gradient update on -- is asked one level up, over pretraining-data domains, in LLM pretraining \citep{xie2023doremi}, a parallel literature at a coarser granularity. \citet{wen2025grades} (GradES) stops updating individual transformer projection matrices once their gradient norm falls below a threshold, a per-parameter-group decision complementary to, rather than competing with, TG-IES's per-instance decision (\cref{app:vit}); we are not aware of a direct empirical comparison in either direction.

\textbf{Curriculum learning.} \citet{bengio2009curriculum} order examples easy-to-hard over training; \citet{soviany2022curriculum} survey the resulting literature. This is complementary to IES/TG-IES: curriculum learning changes when an example enters training, while IES/TG-IES changes when it leaves.

\textbf{Direct predecessor.} \citet{yuan2025instance} is the paper this work theoretically grounds; \cref{sec:intro} details the five gaps closed.

\section{Extended Discussion of Assumptions}
\label{app:assumptions-extended}

\begin{assumption}[Uniform bounds on $c_i$]
\label{ass:cbounds}
There exist $0 < c_{\min} \le c_{\max} < \infty$ with $c_i \in [c_{\min},c_{\max}]$ for all $i$ (\cref{ass:geometric}'s amplitude).
\end{assumption}

\begin{assumption}[Bounded third derivatives]
\label{ass:thirdderiv}
There exists $M_3<\infty$ and a neighborhood $U$ of the GD trajectory such that for all $i$, all $\theta\in U$, and unit vectors $u,v,w$: $|\nabla^3 L_i(\theta)[u,v,w]| \le M_3$.
\end{assumption}

\begin{remark}[Why $\beta$-smoothness, and where it is used]
\Cref{ass:smooth} caps how fast the full gradient can change as parameters move. It is the only regularity assumption used by the actual-trajectory results (\cref{thm:actual,thm:descent}), where it turns a per-step gradient error into a multi-step parameter gap. Because smoothness alone also allows a gap to be amplified from step to step, the resulting bound carries the factor $(1+\beta\eta)^H$; \cref{rem:contraction-direction} discusses the additional contraction condition that would remove it.
\end{remark}

\begin{remark}[Scope of the geometric-trajectory assumption]
\Cref{ass:geometric} is a local, near-convergence approximation, not a global description of training: loss trajectories in practice are non-monotonic and noisy, and $r_i$ varies over time. \Cref{sec:level2}'s quartile analysis finds it a better approximation in the second half of training than the first.
\end{remark}

\begin{remark}[Noise-model realism]
\Cref{ass:noise} is an idealization: SGD revisits overlapping minibatches, correlating noise across nearby $t$; \cref{ass:subgaussian} may also be violated, since deep-learning gradient noise can be heavy-tailed. Under correlated or heavy-tailed noise, \cref{thm:bv}'s variance formula becomes a lower bound and \cref{cor:threshold}'s tail bound may be optimistic. \Cref{sec:noise-validation} tests time-independence directly and finds the measured autocorrelation small and consistent with a detrending-method artifact rather than real correlation, correcting \cref{thm:bv} at $N{=}2$ by under $1.3\%$ even under the conservative reading.
\end{remark}

\begin{remark}[Why small rate dispersion]
$r_i$ is instance $i$'s per-epoch decay rate; \cref{ass:dispersion} says different instances' rates do not spread too far from a common $\bar r$, letting \cref{prop:cv} treat $\CV(X(2))$ as a small perturbation of $\CV(X(1))$ via a first-order Taylor argument valid only when the spread is genuinely small. \Cref{app:counterexample} shows explicitly what breaks when it is not.
\end{remark}

\begin{remark}[Why bounded $c_i$]
$c_i$ is instance $i$'s excess-loss amplitude; bounding it away from $0$ and $\infty$ keeps $\CV(\cdot)$ well-defined and the perturbation constant in \cref{prop:cv} finite, rather than dominated by one pathological instance.
\end{remark}

\begin{remark}[Why bounded third derivatives]
This bounds how quickly curvature itself can change, keeping \cref{thm:curvature}'s Taylor remainder genuinely smaller than the $O(\eta^2)$ curvature term it isolates; without it, ``curvature is $O(\eta^2)$, drift is $O(\eta)$'' would not be a controlled statement.
\end{remark}

\begin{remark}[An unused assumption, kept for future work]
\label{ass:envelope-app}
A bounded gradient envelope on a window ($G := \sup_t\|g^{(t)}\|<\infty$) is a natural companion to $\beta$-smoothness for a future contraction argument bounding trajectory divergence over many steps (\cref{rem:contraction-direction}), but no theorem, corollary, or proof in this paper's current form invokes it. We name it here rather than silently omit it.
\end{remark}

\section{Full Proofs}
\label{app:proofs}

This appendix proves the results of \cref{sec:theory,sec:bridge}. Four statements moved out of the main text for space (\cref{lem:geomdiff,prop:cv,lem:descent,cor:bridgeprob}) are given here next to their proofs. The actual-trajectory results are proved in \cref{app:multistep}, and the single-removal certificate in \cref{app:single-removal}.

\begin{lemma}[Finite differences of a geometric trajectory]
\label{lem:geomdiff}
Under \cref{ass:geometric}, $\Delta^N(L_i^{(t)}-L^*)=c_i\,r_i^{t-N}(r_i-1)^N$.
\end{lemma}

\begin{proof}[Proof of \cref{lem:geomdiff}]
Induction on $N$. \emph{Base case} ($N=0$): $\Delta^0(L_i^{(t)}-L^*) = c_i r_i^t$. \emph{Inductive step.} Assume $\Delta^{N-1}(L_i^{(t)}-L^*) = c_i r_i^{t-(N-1)}(r_i-1)^{N-1}$. Then
\begin{align*}
\Delta^N(L_i^{(t)}-L^*) &= \Delta^{N-1}(L_i^{(t)}-L^*) - \Delta^{N-1}(L_i^{(t-1)}-L^*) \\
&= c_i\, r_i^{t-N+1}(r_i-1)^{N-1} - c_i\, r_i^{t-N}(r_i-1)^{N-1} \\
&= c_i(r_i-1)^{N-1}\big(r_i^{t-N+1}-r_i^{t-N}\big) = c_i\,r_i^{t-N}(r_i-1)^N,
\end{align*}
which is the claim for $N$.
\end{proof}

\begin{proof}[Proof of \cref{thm:ratio}]
By \cref{lem:geomdiff}, $\Delta^N(L_i^{(t)}-L^*) = c_i\, r_i^{t-N}(r_i-1)^N$ and $\Delta^{N-1}(L_i^{(t)}-L^*) = c_i\, r_i^{t-N+1}(r_i-1)^{N-1}$. Dividing,
\[
\frac{\Delta^N(L_i^{(t)}-L^*)}{\Delta^{N-1}(L_i^{(t)}-L^*)} = \frac{r_i^{t-N}}{r_i^{t-N+1}}\cdot(r_i-1) = r_i^{-1}(r_i-1) = \frac{r_i-1}{r_i},
\]
where $c_i$ cancels; the denominator is nonzero whenever $c_i>0$, $r_i\in(0,1)$, $N\ge1$.
\end{proof}

\begin{proposition}[CV closeness under small rate dispersion]
\label{prop:cv}
Under \cref{ass:geometric,ass:dispersion,ass:cbounds}, let $X_i(N) := |c_i r_i^{t-N}(1-r_i)^N|$ for $N\in\{1,2\}$ at fixed $t$, and let $\CV(X):=\sqrt{\Var(X)}/\mathbb E[X]$ across instances $I\sim\mathrm{Unif}\{1,\dots,N\}$. There is a finite constant $C(t,\bar r,c_{\min},c_{\max})$ (explicit form in \cref{app:constants}) such that, for all sufficiently small $\delta_r$,
$\big|\CV(X(2)) - \CV(X(1))\big| \le C(t,\bar r,c_{\min},c_{\max})\cdot\delta_r$.
\end{proposition}

\begin{proof}[Proof of \cref{prop:cv}]
\emph{Step 1 (pointwise perturbation).} Let $f_1(r)=r^{t-1}(1-r)$, $f_2(r)=r^{t-2}(1-r)^2$ on $[\bar r-\delta_r,\bar r+\delta_r]$; both are smooth there, so $L_N := \sup_r|f_N'(r)| < \infty$. Writing $r_i=\bar r+\epsilon_i$, $|\epsilon_i|\le\delta_r$, and $\Delta_i^{(N)} := f_N(r_i)-f_N(\bar r)$, the mean value theorem gives $|\Delta_i^{(N)}| \le L_N\delta_r$. With $Y_i^{(N)} := c_i f_N(\bar r)$, $E_i^{(N)} := c_i\Delta_i^{(N)}$: $X_i(N) = Y_i^{(N)} + E_i^{(N)}$, $|E_i^{(N)}| \le c_{\max}L_N\delta_r$.

\emph{Step 2 (mean perturbation).} $|\mathbb E[X(N)]-\mathbb E[Y(N)]| \le c_{\max}L_N\delta_r$.

\emph{Step 3 (variance perturbation).} $X_i(N)^2 - Y_i(N)^2 = E_i^{(N)}(2Y_i^{(N)} + E_i^{(N)})$ is $O(\delta_r)$ uniformly, so $|\Var(X(N)) - \Var(Y(N))| \le (A_N+B_N)\delta_r$ for finite $A_N,B_N$, hence $|\sqrt{\Var(X(N))} - \sqrt{\Var(Y(N))}| \le D_N\delta_r$.

\emph{Step 4 (Lipschitz bound on CV).} $\phi(\mu,s)=s/\mu$ is $C^1$ with bounded gradient $\Lambda_N$ on the relevant compact region, so $|\CV(X(N)) - \CV(Y(N))| \le \Lambda_N(c_{\max}L_N + D_N)\,\delta_r$.

\emph{Step 5.} $\CV(Y(N)) = \sqrt{\Var(c)}/\mathbb E[c]$ is identical for $N=1,2$ ($f_N(\bar r)$ cancels in the ratio); applying Step 4 to both and the triangle inequality gives the claim with $C := \max_N\Lambda_N(c_{\max}L_N+D_N)$, made explicit in \cref{app:constants}.
\end{proof}

\begin{proof}[Proof of \cref{thm:bv}]
By linearity, $\Delta^N\tilde L_i^{(t)} = a+b$ with $a:=\Delta^N L_i^{(t)} = c_i r_i^{t-N}(r_i-1)^N$ (deterministic, \cref{lem:geomdiff}) and $b:=\Delta^N\xi_{i,t} = \sum_{k=0}^N(-1)^k\binom Nk\xi_{i,t-k}$. Then $\mathbb E[(\Delta^N\tilde L_i^{(t)})^2] = a^2 + 2a\mathbb E[b] + \mathbb E[b^2]$. By \cref{ass:noise}, $\mathbb E[b]=0$; since the $\xi_{i,t-k}$ are independent with variance $\sigma^2$, $\mathbb E[b^2] = \sigma^2\sum_k\binom Nk^2 = \sigma^2\binom{2N}N$ (Vandermonde). Combining gives \cref{eq:bv}.
\end{proof}

\begin{proof}[Proof of \cref{prop:mse}]
Write $\alpha_i := (1-r_i)^2/r_i^2$, so $r_i^{2(t-N)}(r_i-1)^{2N} = r_i^{2t}\alpha_i^N$ and $\mathrm{MSE}(N,t) = c_i^2 r_i^{2t}\alpha_i^N + \binom{2N}N\sigma^2$. Using $\binom21=2,\binom42=6,\binom63=20$, $\mathrm{MSE}(1,t)-\mathrm{MSE}(2,t) = c_i^2r_i^{2t}\alpha_i(1-\alpha_i)-4\sigma^2$ and $\mathrm{MSE}(3,t)-\mathrm{MSE}(2,t) = 14\sigma^2 - c_i^2r_i^{2t}\alpha_i^2(1-\alpha_i)$. Hence
\begin{align}
\mathrm{MSE}(2,t)\le\mathrm{MSE}(1,t) &\iff c_i^2r_i^{2t}\alpha_i(1-\alpha_i)\ge4\sigma^2, \label{eq:mse12}\\
\mathrm{MSE}(2,t)\le\mathrm{MSE}(3,t) &\iff c_i^2r_i^{2t}\alpha_i^2(1-\alpha_i)\le14\sigma^2. \label{eq:mse23}
\end{align}
For $\alpha_i<1$ ($r_i>\tfrac12$), $\alpha_i(1-\alpha_i)>0$, so dividing \cref{eq:mse12,eq:mse23} by it gives the two sides of \cref{eq:band}; $N=2$ minimizes over $\{1,2,3\}$ exactly when both hold. The band is non-empty since $4\alpha_i\le14$ for every $\alpha_i<1$.
\end{proof}

\begin{proof}[Proof of \cref{cor:threshold}]
By \cref{ass:noise}, $\Var(\Delta^2\xi_{i,t}) = \Var(\xi_{i,t}-2\xi_{i,t-1}+\xi_{i,t-2}) = \sigma^2(1+4+1) = 6\sigma^2$, matching \cref{thm:bv} at $N=2$. The Gaussian approximation is standard. Under \cref{ass:subgaussian}, $\Delta^2\xi_{i,t}$ is a linear combination of independent $\sigma^2$-sub-Gaussian variables with squared coefficients summing to $6$, hence $6\sigma^2$-sub-Gaussian; the standard tail inequality gives $2\exp(-k^2/2)$.
\end{proof}

\begin{proof}[Proof of \cref{thm:curvature}]
Let $\Delta\theta^{(s)}=-\eta g^{(s)}$ (full-batch). Third-order Taylor expansion with remainder (\cref{ass:thirdderiv}):
\[
L_i(\theta^{(s)}+\Delta\theta^{(s)}) = L_i^{(s)} + \langle g_i^{(s)},\Delta\theta^{(s)}\rangle + \tfrac12\Delta\theta^{(s)\top}H_i(\theta^{(s)})\Delta\theta^{(s)} + r_{i,s}, \quad |r_{i,s}|\le\tfrac{M_3}6\eta^3\|g^{(s)}\|^3.
\]
Applying at $s=t-1,t-2$ and forming $\Delta^2 L_i^{(t)} = [L_i^{(t)}-L_i^{(t-1)}]-[L_i^{(t-1)}-L_i^{(t-2)}]$, the $O(\eta)$ terms combine into the drift term and the $O(\eta^2)$ terms into the curvature difference; $R_{i,t}:=r_{i,t-1}-r_{i,t-2}$ and the triangle inequality give the stated remainder bound.
\end{proof}

\begin{lemma}[Smoothness bounds the gradient by suboptimality]
\label{lem:descent}
If $f$ is $\Lambda$-smooth on $U$ and $\theta^*\in U$ satisfies $f(\theta^*)\le f(\theta-\tfrac1\Lambda\nabla f(\theta))$ for $\theta\in U$, then $\|\nabla f(\theta)\|^2 \le 2\Lambda(f(\theta)-f(\theta^*))$.
\end{lemma}

\begin{proof}[Proof of \cref{lem:descent}]
By $\Lambda$-smoothness (descent lemma), $f(\theta-\tfrac1\Lambda\nabla f(\theta)) \le f(\theta) - \tfrac1{2\Lambda}\|\nabla f(\theta)\|^2$. Combined with $f(\theta^*)\le f(\theta-\tfrac1\Lambda\nabla f(\theta))$: $f(\theta^*) \le f(\theta)-\tfrac1{2\Lambda}\|\nabla f(\theta)\|^2$, which rearranges to the claim.
\end{proof}

\begin{proof}[Proof of \cref{thm:bridge}]
By \cref{lem:descent}, $\|g_i^{(t)}\|^2 \le 2\Lambda_i(L_i^{(t)}-L^*) = 2\Lambda_i c_i r_i^t$ (\cref{ass:geometric}). By \cref{lem:geomdiff} at $N=2$, $\Delta^2 L_i^{(t)} = c_i r_i^{t-2}(1-r_i)^2$, so $c_i r_i^t = (r_i/(1-r_i))^2\Delta^2 L_i^{(t)}$. Substituting gives $\|g_i^{(t)}\|^2 \le 2\Lambda_i(r_i/(1-r_i))^2\Delta^2 L_i^{(t)}$, and taking square roots gives \cref{eq:bridge}.
\end{proof}

\begin{corollary}[Probabilistic bridge under observation noise]
\label{cor:bridgeprob}
Under the assumptions of \cref{thm:bridge} and \cref{ass:subgaussian}, for $k>0$ and $\delta':=\delta+k\sqrt6\,\sigma$: $\Pr\big(\hat S_i<\delta \text{ and } \|g_i^{(t)}\|\ge\varepsilon_i(\delta')\big) \le 2\exp(-k^2/2)$.
\end{corollary}

\begin{remark}[Scope of the bridge]
\label{rem:bridge-scope}
Two caveats apply. First, $\varepsilon_i(\delta)\propto r_i/(1-r_i)$ grows without bound as $r_i\to1$, so the bound is uninformative for near-stationary, as opposed to near-converged, instances. Second, \cref{cor:bridgeprob} bounds a \emph{joint}, not a conditional, probability, since $\hat S_i$ and $\Delta^2\xi_{i,t}$ share the same noise.
\end{remark}

\begin{proof}[Proof of \cref{cor:bridgeprob}]
$\Delta^2 L_i^{(t)} = \hat S_i \mp \Delta^2\xi_{i,t}$, so $\hat S_i<\delta$ with $|\Delta^2\xi_{i,t}|\le k\sqrt6\sigma$ implies $\Delta^2 L_i^{(t)} < \delta'$, hence $\|g_i^{(t)}\|<\varepsilon_i(\delta')$ by \cref{thm:bridge}. So the event in the statement is contained in $\{|\Delta^2\xi_{i,t}|>k\sqrt6\sigma\}$, bounded by \cref{cor:threshold} at $2\exp(-k^2/2)$.
\end{proof}

\section{Actual-Trajectory Analysis: Proofs and Discussion}
\label{app:multistep}
\label{app:actual}

This appendix proves \cref{thm:actual,thm:descent}, recovers the per-instance multi-step bound of our earlier analysis as a special case, and discusses what would remove the exponential factor. Throughout, $\theta^{(t)}$ is full-data gradient descent, $\hat\theta^{(t)}$ is the TG-IES trajectory (\cref{eq:activegrad}), $A_t$ and $P_t$ are the algorithm's own evolving active and passive sets, and $\omega_t$ is the omitted-gradient residual of \cref{eq:inexact}. As in the main text, the analysis idealizes each epoch as one full-batch step.

\begin{proof}[Proof of \cref{thm:actual}]
Let $e_t:=\hat\theta^{(t)}-\theta^{(t)}$, with $e_{t_0}=0$. By \cref{eq:inexact},
\[
e_{t+1}=e_t-\eta_t\big(g(\hat\theta^{(t)})-g(\theta^{(t)})\big)+\eta_t\,\omega_t(\hat\theta^{(t)}).
\]
By \cref{ass:smooth}, $\|g(\hat\theta^{(t)})-g(\theta^{(t)})\|\le\beta\|e_t\|$, so the triangle inequality and the residual bound give
\begin{equation}
\label{eq:error-recursion}
\|e_{t+1}\|\le(1+\beta\eta_t)\|e_t\|+\eta_t\rho_t .
\end{equation}
Unrolling \cref{eq:error-recursion} from $t_0$ to $t_0+H$ with $\|e_{t_0}\|=0$ gives \cref{eq:actual}, where an empty product equals $1$. For constant $\eta$ and $\rho_t\le\rho$, the right side of \cref{eq:actual} is at most $\eta\rho\sum_{s=0}^{H-1}(1+\beta\eta)^{H-1-s}=\eta\rho\,\frac{(1+\beta\eta)^H-1}{\beta\eta}$, which is the constant-step bound stated in \cref{thm:actual}. No convexity is used.
\end{proof}

\begin{proof}[Proof of \cref{thm:descent}]
Write $g_t:=g(\hat\theta^{(t)})$ and $\omega_t:=\omega_t(\hat\theta^{(t)})$, so that $g_t=\hat g^{(t)}+\omega_t$ by \cref{eq:inexact}. The descent lemma for the $\beta$-smooth loss $L$, applied to the step $\hat\theta^{(t+1)}=\hat\theta^{(t)}-\eta_t\hat g^{(t)}$, gives
\[
L(\hat\theta^{(t+1)})\le L(\hat\theta^{(t)})-\eta_t\langle g_t,\hat g^{(t)}\rangle+\frac{\beta\eta_t^2}{2}\|\hat g^{(t)}\|^2 .
\]
By Cauchy--Schwarz and the relative-error condition, $\langle g_t,\hat g^{(t)}\rangle=\|\hat g^{(t)}\|^2+\langle\omega_t,\hat g^{(t)}\rangle\ge\|\hat g^{(t)}\|^2-\|\omega_t\|\,\|\hat g^{(t)}\|\ge(1-\nu)\|\hat g^{(t)}\|^2$. Substituting gives \cref{eq:descent}. The coefficient $1-\nu-\beta\eta_t/2$ is positive exactly when $\eta_t<2(1-\nu)/\beta$.
\end{proof}

\textbf{The per-instance route as a special case.} An earlier version of this analysis bounded the same trajectory gap by assuming a gradient floor for every passive instance. That result follows from \cref{thm:actual} through the triangle inequality in \cref{eq:aggregate}.

\begin{corollary}[Per-instance multi-step bound]
\label{prop:multistep}
Under \cref{ass:smooth} with a constant step size $\eta$, suppose every passive instance satisfies $\|g_i(\hat\theta^{(t)})\|<\varepsilon_i$ at every step along the actual TG-IES trajectory. Let $\bar\varepsilon_t:=\frac1N\sum_{i\in P_t}\varepsilon_i$ and $\varepsilon_{\max}:=\max_i\varepsilon_i$. Then, with $t_0=0$,
\begin{equation}
\label{eq:multistep}
\|\hat\theta^{(t)}-\theta^{(t)}\|\;\le\;\eta\sum_{s=0}^{t-1}(1+\eta\beta)^{t-1-s}\,\bar\varepsilon_s\;\le\;\frac{\varepsilon_{\max}}{\beta}\Big[(1+\eta\beta)^t-1\Big].
\end{equation}
\end{corollary}

\begin{proof}
$\|\omega_s(\hat\theta^{(s)})\|\le\frac1N\sum_{i\in P_s}\|g_i(\hat\theta^{(s)})\|<\bar\varepsilon_s$, so \cref{eq:actual} applies with $\rho_s=\bar\varepsilon_s$. The second inequality uses $\bar\varepsilon_s\le(|P_s|/N)\,\varepsilon_{\max}\le\varepsilon_{\max}$.
\end{proof}

The aggregate condition of \cref{thm:actual} is weaker than the per-instance condition of \cref{prop:multistep}. It constrains only the norm of the \emph{summed} passive gradient, so it can hold even when many individual passive gradients exceed a per-instance floor, for example when those gradients partly cancel. The reactivation probe (\cref{app:reactprobe}) shows that the per-instance condition fails for about half of the sampled passive instances at any snapshot, which is why the main text states the aggregate version. The probe does not measure $\omega_t$, however, so the aggregate condition remains an assumption.

\begin{remark}[What would remove the exponential factor]
\label{rem:contraction-direction}
The factor $(1+\beta\eta)^H$ in \cref{thm:actual} arises because smoothness alone lets the gradient-descent map expand parameter gaps. The absolute divergence measured in \cref{sec:kstartraj} instead \emph{plateaus}, which is qualitatively more consistent with a contractive map. If $L$ were additionally $\mu$-strongly convex on a convex region containing both trajectories and $\eta<2/\beta$, the map $\theta\mapsto\theta-\eta\,g(\theta)$ would be a contraction with factor $\kappa=\max(|1-\eta\mu|,|1-\eta\beta|)<1$. Replacing $(1+\beta\eta)$ by $\kappa$ in \cref{eq:error-recursion} gives the horizon-independent bound $\|\hat\theta^{(t)}-\theta^{(t)}\|\le\eta\rho/(1-\kappa)$, which matches the observed plateau qualitatively. Local contraction of this kind is plausible in restricted regimes (late training, over-parameterized models, near a broad minimum), but we have not verified it along the trajectories of our architectures, and we flag it as the natural next step.
\end{remark}

\section{The Single-Removal Reference Certificate}
\label{app:single-removal}
Our earlier analysis bounded the effect of removing a single instance under a shared-reference idealization. We keep it here because it yields the closed-form horizon $K_{\mathrm{safe}}$ used to initialize $K^*$. It is no longer the main stability result: \cref{thm:actual} is stated for the actual trajectories and contains this certificate as its first-order ($\beta\to0$) limit for a single passive instance.

\begin{theorem}[Single-removal certificate, shared-reference trajectory]
\label{thm:safe}
Suppose instance $i$ satisfies $\|g_i^{(t)}\| < \varepsilon$ for all $t\in[T-K,T]$. Consider two runs that share the same reference trajectory $\{\theta^{(t)}\}_{t=T-K}^T$, both use \cref{eq:activegrad} evaluated on it, and differ only in whether $i \in A_t$. Then $\big\|\theta^{(T)}_{\mathrm{full}} - \theta^{(T)}_{\mathrm{rem}}\big\| \le K\eta\varepsilon/N$. If the test loss is additionally $\beta_{\mathrm{test}}$-Lipschitz, then $\big|L_{\mathrm{test}}(\theta^{(T)}_{\mathrm{full}}) - L_{\mathrm{test}}(\theta^{(T)}_{\mathrm{rem}})\big| \le \beta_{\mathrm{test}} K\eta\varepsilon/N$.
\end{theorem}

\begin{corollary}[Conditional reference horizon]
\label{cor:kstar}
Given a tolerance $\tau>0$ on the test-loss deviation, any exclusion window $K \le K_{\mathrm{safe}} := \lfloor N\tau/(\beta_{\mathrm{test}}\,\eta\,\varepsilon)\rfloor$ keeps the deviation in \cref{thm:safe} at most $\tau$. The floor, not the ceiling, is required, since rounding up would violate the bound.
\end{corollary}

\begin{remark}[Why this certificate is conditional and local]
\label{rem:shared-traj}
\Cref{thm:safe} compares two runs evaluated on the \emph{same} reference trajectory. In an actual run, the two trajectories separate from the moment $i$ is removed, and the separation feeds back through the gradient. \Cref{thm:actual} accounts for this feedback through the factor $(1+\beta\eta)$, and \cref{thm:safe} is its $\beta\to0$ limit for a single passive instance. The certificate is also local: it covers one removal, not the repeated and overlapping removals of \cref{alg:tgies}.
\end{remark}

\begin{remark}[$K_{\mathrm{safe}}$ is not the algorithm's patience]
\label{rem:kmismatch}
$K_{\mathrm{safe}}$ bounds a \emph{closed} exclusion window counted backward from a fixed endpoint, whereas the patience $K^*$ of \cref{alg:tgies} governs an \emph{open-ended} exclusion that begins after $K^*$ consecutive passing epochs. Nothing here shows that initializing $K^*$ from $K_{\mathrm{safe}}$ keeps that open-ended exclusion within the budget of \cref{cor:kstar}. In addition, \cref{thm:safe} requires $\|g_i^{(t)}\|<\varepsilon$, whereas the algorithm thresholds the noisy signal $\hat S_i<\delta$; \cref{thm:bridge} connects the two only under \cref{ass:smoothi}. We therefore call $K_{\mathrm{safe}}$ a conditional reference horizon and $K^*$ a theory-informed persistence rule.
\end{remark}

\begin{proof}[Proof of \cref{thm:safe}]
At each $t\in[T-K,T]$, both runs use \cref{eq:activegrad} on the same reference trajectory, so their update directions differ only by the term $\frac1N g_i^{(t)}$. Hence each step adds at most $\frac{\eta}{N}\|g_i^{(t)}\| \le \frac{\eta\varepsilon}{N}$ to the gap, and summing over $K$ steps (triangle inequality) gives $\|\theta^{(T)}_{\mathrm{full}} - \theta^{(T)}_{\mathrm{rem}}\| \le K\eta\varepsilon/N$. The test-loss bound follows from $\beta_{\mathrm{test}}$-Lipschitzness.
\end{proof}

\begin{proof}[Proof of \cref{cor:kstar}]
$\beta_{\mathrm{test}} K \eta \varepsilon/N \le \tau \iff K \le N\tau/(\beta_{\mathrm{test}}\eta\varepsilon)$. The largest integer satisfying this is $K_{\mathrm{safe}} = \lfloor N\tau/(\beta_{\mathrm{test}}\eta\varepsilon)\rfloor$; a ceiling would round up past the constraint.
\end{proof}

\section{Explicit Constants for \texorpdfstring{\cref{prop:cv}}{the CV Proposition}}
\label{app:constants}

Recall from \cref{ass:cbounds} that $\Var(c) \le (c_{\max}-c_{\min})^2/4 =: V_c$. One concrete choice of the constant in \cref{prop:cv} is
\[
C(t,\bar r,c_{\min},c_{\max}) := \max_{N\in\{1,2\}}\left(\frac{c_{\max}L_N}{c_{\min}f_N(\bar r)} + \frac{c_{\max}L_N}{\sqrt{V_c}\,c_{\min}f_N(\bar r)+10^{-12}}\right),
\]
where $f_1(r)=r^{t-1}(1-r)$, $f_2(r)=r^{t-2}(1-r)^2$, $L_N=\sup_{r\in[\bar r-\delta_r,\bar r+\delta_r]}|f_N'(r)|$. The $10^{-12}$ guards against division by zero when all $c_i$ are equal ($V_c=0$).

\section{Counterexample: Global CV Monotonicity Can Fail}
\label{app:counterexample}

A global claim $\CV(\Delta^2 L) \le \CV(\Delta^1 L)$ (without \cref{ass:dispersion}) can fail under wide rate dispersion. Take $c_1=c_2=1$, $L^*=0$, $r_1=0.01$, $r_2=0.99$, $t=10$:
\[
\Delta^1(L_1^{(10)})\approx-10^{-18},\ \Delta^1(L_2^{(10)})\approx-0.0091,\qquad
\Delta^2(L_1^{(10)})\approx10^{-16},\ \Delta^2(L_2^{(10)})\approx9.2{\times}10^{-5}.
\]
The cross-instance magnitude ratio is order $10^{16}$ at $N{=}1$ and order $10^{11}$ at $N{=}2$: dispersion (CV across the two instances) \emph{increases}, not decreases, from $N{=}1$ to $N{=}2$ here, confirming \cref{ass:dispersion} is load-bearing, not a technicality.

\section{Extended Practical Implementation Notes}
\label{app:practical}

\textbf{Recalibration schedule.} $\delta$ and $K^*$ are recalibrated once, at epoch 30, then held fixed, rather than repeating every $\Delta$ epochs: repeating was found counterproductive, since a shrinking noise floor $\hat\sigma$ re-tightens $\delta$ over time and every recalibration resets in-progress $K^*$-window counts, together strangling savings. This is an empirical finding, not a theorem (\cref{sec:limitations}).

\textbf{Gradient-norm floor estimate.} $\varepsilon_0$ is a low percentile (default 20th) of a cheap last-layer proxy $\|\mathrm{softmax}(z_i)-\mathrm{onehot}(y_i)\|\cdot\|\mathrm{feat}_i\|$ over a subsample of the active set (capped at 8192 instances), cutting recalibration cost roughly $6\times$ versus a full backward pass without materially changing the percentile estimate.

\textbf{Consistent forward-pass discipline.} Following \citet{yuan2025instance}'s stated Algorithm~1 rather than the shipped implementation, all instances (active and passive) are forwarded in the same \texttt{train()} mode every epoch (passive under \texttt{no\_grad}), avoiding a train/eval batch-normalization mismatch \citep{ioffe2015batch}.

\textbf{A rounding-direction correction, with no effect on reported results.} \Cref{cor:kstar} rounds $K_{\mathrm{safe}}$ down; the implementation generating every run in this paper used a ceiling instead, caught only while finalizing this version. Checking every run's logged $K^*$ trajectory: in every case $K^*$ reaches $\mathrm{max\_kstar}=15$ immediately upon recalibration, i.e., the uncapped value already exceeds the clamp before rounding is applied, so ceiling vs.\ floor changes nothing here. The implementation now uses the floor consistent with \cref{cor:kstar}. The same check shows that the operational patience is $K^*=K_{\max}=15$ in every reported TG-IES run, which is why the main text describes $K^*$ as a theory-informed persistence rule rather than a derived safety horizon.

\section{Full Zero-Shot Transfer Table}
\label{app:zeroshot}

\begin{table}[h]
\caption{TG-IES, every distinct (dataset, architecture, optimizer/schedule) cell in this paper, same four fixed calibration constants throughout. Mean $\pm$ std over seeds where $n{>}1$; asterisked rows are single-seed by design (\cref{sec:lrmechanism}'s matched-epoch schedule comparison).}
\label{tab:zeroshot}
\centering
\small
\begin{tabular}{lccc}
\toprule
Cell & $n$ & Acc.\ (\%) & Saved (\%) \\
\midrule
CIFAR-10 / ResNet-18 / SGD-exp (primary) & 3 & $94.85\pm0.28$ & $53.63\pm0.30$ \\
CIFAR-10 / ResNet-18 / SGD-fixed & 3 & $91.91\pm0.18$ & $29.27\pm0.30$ \\
CIFAR-10 / ResNet-18 / SGD-linear & 1 & $94.85^*$ & $32.05^*$ \\
CIFAR-10 / ResNet-18 / SGD-step & 1 & $94.44^*$ & $43.52^*$ \\
CIFAR-10 / ResNet-18 / SGD-cosine & 1 & $94.73^*$ & $23.15^*$ \\
CIFAR-10 / ResNet-18 / Adam & 3 & $92.44\pm0.10$ & $30.18\pm0.62$ \\
CIFAR-10 / VGG-16 / SGD-exp & 3 & $93.11\pm0.14$ & $54.43\pm0.12$ \\
CIFAR-10 / DenseNet-121 / SGD-exp & 3 & $95.28\pm0.09$ & $54.53\pm0.66$ \\
CIFAR-100 / ResNet-18 / SGD-exp & 3 & $76.01\pm0.08$ & $57.76\pm0.04$ \\
CIFAR-10 / ViT-Tiny / AdamW-wc & 3 & $83.71\pm0.47$ & $25.52\pm0.26$ \\
CIFAR-100 / ViT-Tiny / AdamW-wc & 3 & $58.56\pm0.47$ & $26.28\pm0.14$ \\
\bottomrule
\end{tabular}
\end{table}

The savings figure varies substantially across cells ($23$--$58\%$), as it should: schedule-, dataset-, and architecture-dependent mechanisms (\cref{sec:lrmechanism,app:adam,app:vit}) directly cause this variation. What does not vary is the calibration procedure: this table was assembled with no per-cell constant lookup, because there is none to look up.

\section{Second Optimizer: Adam}
\label{app:adam}

\begin{table}[h]
\caption{ResNet-18, CIFAR-10, Adam, mean $\pm$ std, 3 seeds.}
\label{tab:adam}
\centering
\small
\begin{tabular}{lcc}
\toprule
Method & Test Acc.\ (\%) & Backprop Saved (\%) \\
\midrule
Baseline & $92.91\pm0.06$ & $0.03\pm0.00$ \\
IES (shipped) & $92.61\pm0.10$ & $56.02\pm0.22$ \\
IES (literal) & $92.42\pm0.10$ & $\mathbf{60.09\pm0.29}$ \\
TG-IES (ours) & $92.44\pm0.10$ & \lose{$30.18\pm0.62$} \\
\bottomrule
\end{tabular}
\end{table}

Adam \citep{kingma2015adam} is a second, more pronounced instance of the fixed-LR pattern of \cref{sec:settings}: both IES variants save more ($56.0$--$60.1\%$) than TG-IES's $30.2\%$, at accuracy within $0.02$--$0.17$ points of TG-IES's (\cref{app:stats}: the accuracy gap vs.\ IES (literal) is not distinguishable from zero at $n=3$, while the saved-\% deficit is consistent in sign across all 3 seeds). Adam and fixed-LR SGD share the structural property distinguishing them from every cell where TG-IES leads: neither has the exponential schedule's large, steadily-decaying learning rate. We report this as an open empirical pattern, not an explained one; \cref{sec:lrmechanism} proposes and checks a candidate mechanism against the logged record.

\section{Toward a Mechanism: The Drift Term Under Non-Decaying Learning Rates}
\label{sec:lrmechanism}

\Cref{thm:curvature} decomposes $\Delta^2 L_i^{(t)} = \eta_t D_i^{(t)} + O(\eta_t^2) + R_{i,t}$, $D_i^{(t)} := \langle g_i^{(t-2)},g^{(t-2)}\rangle - \langle g_i^{(t-1)},g^{(t-1)}\rangle$. Under the exponential schedule, $\eta_t\to0$ drives the whole drift term toward zero regardless of $D_i^{(t)}$'s own decay; under a non-decaying schedule (fixed-LR SGD, or Adam's lack of scalar-$\eta$ decay), only $D_i^{(t)}\to0$ on its own can shrink it, which \cref{ass:geometric}'s geometric picture gives no particular reason to expect on any fixed timetable. This predicts $\delta$ should calibrate smaller wherever $\eta$ is smaller, and fewer instances should reach $K^*{=}15$ consecutive low-signal epochs. Both hold in the logged data: at epoch $31$, $\delta$ recalibrates to $0.01597$ (exponential, $\eta_{31}{\approx}0.0282$) vs.\ $0.00247$ (fixed-LR, $6.5\times$ smaller) and $0.000257$ (Adam, $62\times$ smaller); by training's end the exponential run excludes $87.4\%$ of instances vs.\ $42.0\%$ (fixed-LR) and $45.6\%$ (Adam), matching the backprop-saved gap.

\textbf{Direct instrumentation of $D_i^{(t)}$.} A diagnostic run computing $g^{(t)}$, every $L_i^{(t)}$, and $\langle g_i^{(t)},g^{(t)}\rangle$ for 300 fixed instances, baseline training, all three schedules, \emph{complicates} the ``drift dominates'' reading: $\eta_tD_i^{(t)}$ correlates only weakly with the observed $\Delta^2L_i^{(t)}$ (Pearson $r=0.34,0.37,0.23$ for exponential/fixed/Adam), and its magnitude does not dominate the combined curvature-plus-remainder term (ratio $0.69\times$, $0.32\times$, and effectively $0\times$ respectively) -- the opposite of a literal ``drift dominates'' reading, and it weakens further in the last training quarter (\cref{fig:driftvalidation}). We read this as refining, not refuting, \cref{thm:curvature}: the $O(\eta)$-vs-$O(\eta^2)$ ordering is asymptotic as $\eta\to0$, and at this network's actual step sizes ($\eta$ up to $0.1$) the curvature term is evidently not negligible in absolute terms. The qualitative argument above only requires drift to \emph{scale} with $\eta_t$, which this does not contradict; but $\Delta^2L_i^{(t)}$ being \emph{well-approximated by} $\eta_tD_i^{(t)}$ alone is not supported, and \cref{sec:level3}'s framing is corrected accordingly.

\begin{figure}[h]
\centering
\includegraphics[width=0.95\linewidth]{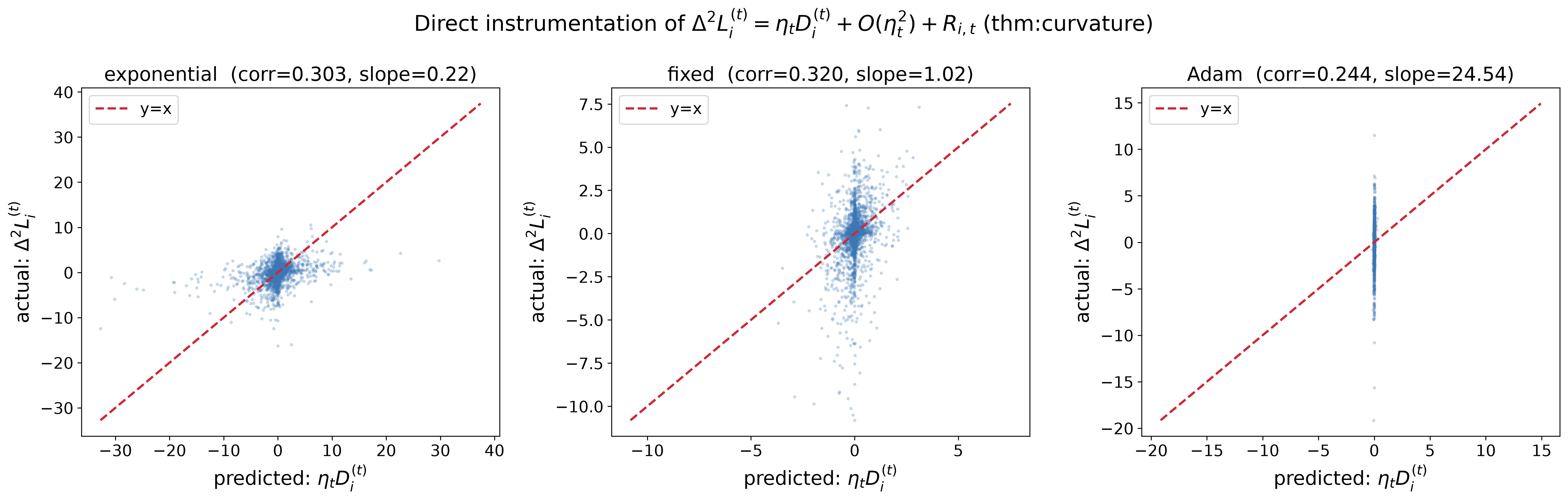}
\caption{Direct instrumentation of $\Delta^2 L_i^{(t)}=\eta_tD_i^{(t)}+O(\eta_t^2)+R_{i,t}$, fixed 300-instance subset, baseline training, all valid epochs, one panel per schedule. Dashed line: $y=x$. The drift term correlates weakly with the observed signal and does not dominate the residual under any schedule.}
\label{fig:driftvalidation}
\end{figure}

\section{Full New-Baselines Comparison and Fixed-Compute-Budget Analysis}
\label{app:newbaselines}
\label{app:equalcompute}

We implement and run four baselines not previously benchmarked against IES: InfoBatch \citep{qin2024infobatch} ($r{=}0.5$ soft-pruning below epoch-mean loss, $1/(1-r)$ gradient rescaling, annealing to the full dataset for the final $12.5\%$ of epochs); BLS-InfoBatch \citep{zhou2026bls} (same schedule, EMA-smoothed score, $\alpha{=}0.7$); AlignPrune \citep{qin2026alignprune} (Dynamic Alignment Score against the population-mean trajectory, since no clean-labeled reference subset exists here); and Selective Backprop \citep{jiang2019accelerating} (percentile-rank accept probability $p^{\beta_{\mathrm{sb}}-1}$, $\beta_{\mathrm{sb}}{=}3$, on the persistent last-observed-loss score, no gradient rescaling and no annealing back to the full dataset, since SB is an explicitly biased sampler by design rather than InfoBatch's unbiased correction). We do not compare against DPPO \citep{zhu2026dppo} (a different, RL setting).

On VGG-16 (seed 0) and CIFAR-100 (seed 0), the same ranking as the primary cell holds among the first three: TG-IES leads at $54.4\%$/$57.8\%$ saved against InfoBatch's $38.9\%$/$35.2\%$, BLS-InfoBatch's $36.4\%$/$30.9\%$, and AlignPrune's $22.7\%$/$22.5\%$, all within 1--2 accuracy points. AlignPrune saves the least of these methods consistently, expected rather than a weakness of our implementation: its Dynamic Alignment Score identifies samples \emph{misaligned} with the population trend (a label-noise signal), not samples converged \emph{with} it (a mastery signal); on clean CIFAR labels most of what it would need to prune to match TG-IES are exactly the well-aligned instances its score is built to keep.

\textbf{Selective Backprop: the other extreme of the trade-off.} On the primary cell (3 seeds), SB saves $66.72\pm0.01\%$ of backpropagation instances -- more than every other method in this paper, TG-IES included ($53.63\%$) -- but at $93.31\pm0.11\%$ accuracy, a $1.72$-point cost relative to baseline ($95.03\%$), the largest accuracy cost of any method compared anywhere in this paper (next largest: TG-IES's own $0.87$ points on CIFAR-100, \cref{tab:overview}). SB's savings are also the most seed-stable of any method measured ($\pm0.01$ points, matching InfoBatch's $\pm0.01$). This is not a failure of our implementation: SB's probabilistic accept/reject rule has no noise-calibrated threshold and no persistence rule (\cref{sec:related}), so nothing in it stops pruning before accuracy degrades. Read together with \cref{tab:primary}, SB and TG-IES mark the two ends of a Pareto-like frontier on this cell: the most raw compute saved for the most accuracy given up, versus a small, bounded accuracy cost for a smaller but still leading (among methods with any comparable accuracy) compute saving.

\begin{figure}[h]
\centering
\includegraphics[width=0.95\linewidth]{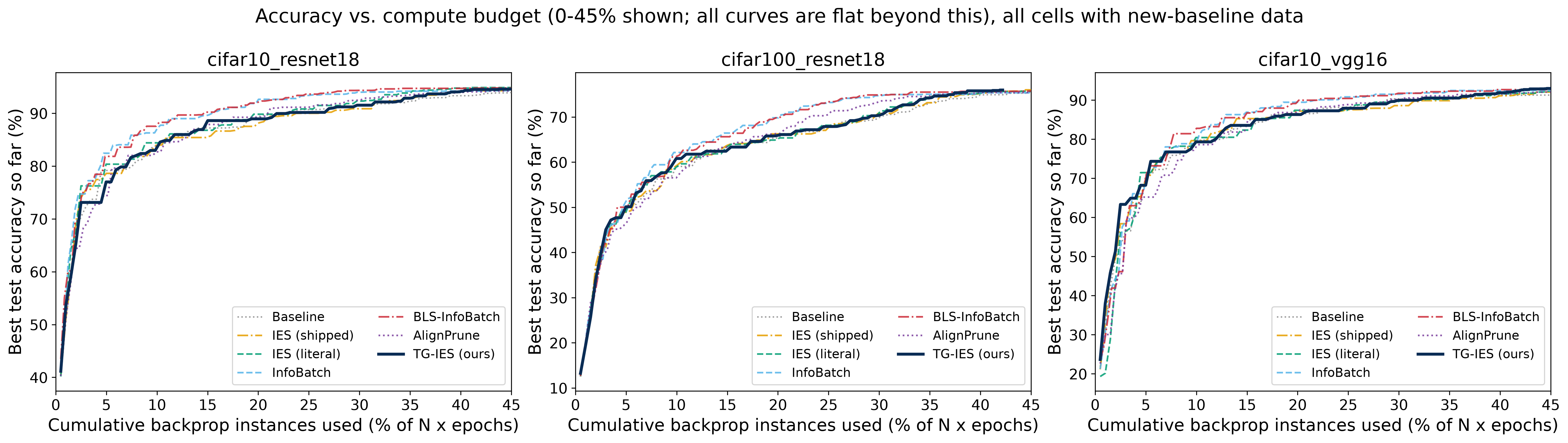}
\caption{Best-so-far test accuracy vs.\ cumulative compute budget (\% of $N\times$total-epochs), all three cells with new-baseline data. A method's curve stops at its own final compute point.}
\label{fig:equalcompute}
\end{figure}

\textbf{Accuracy under a fixed compute budget.} At a common compute budget rather than each method's own stopping point, the result is a genuine trade-off, consistent across all three cells: at low budgets ($20$--$30\%$), InfoBatch and BLS-InfoBatch lead TG-IES by $2.5$--$4.3$ points, since they skip both forward and backward on a pruned instance and so get more model updates per unit of measured compute early in training. This narrows to $0.1$--$0.8$ points by $40\%$ and reverses on CIFAR-100 ($+0.12$ points for TG-IES). The practical implication is budget-dependent: a practitioner constrained to $20$--$30\%$ of full compute would currently do better with InfoBatch-family methods; the own-stopping-point comparison of \cref{tab:primary} describes a different, later operating regime. We report both framings rather than hide the trade-off.

\section{Empirical Validation of the Noise-Independence Assumption}
\label{sec:noise-validation}

\Cref{ass:noise} assumes $\xi_{i,t}$ independent across $t$; we test this directly. Detrending each instance's epoch-101--200 loss trajectory (ResNet-18/CIFAR-10, seed 0) with an OLS linear fit and pooling residuals across instances estimates lag-$k$ autocorrelation directly (\cref{fig:noiseautocorr}).

\begin{figure}[h]
\centering
\includegraphics[width=0.4\linewidth]{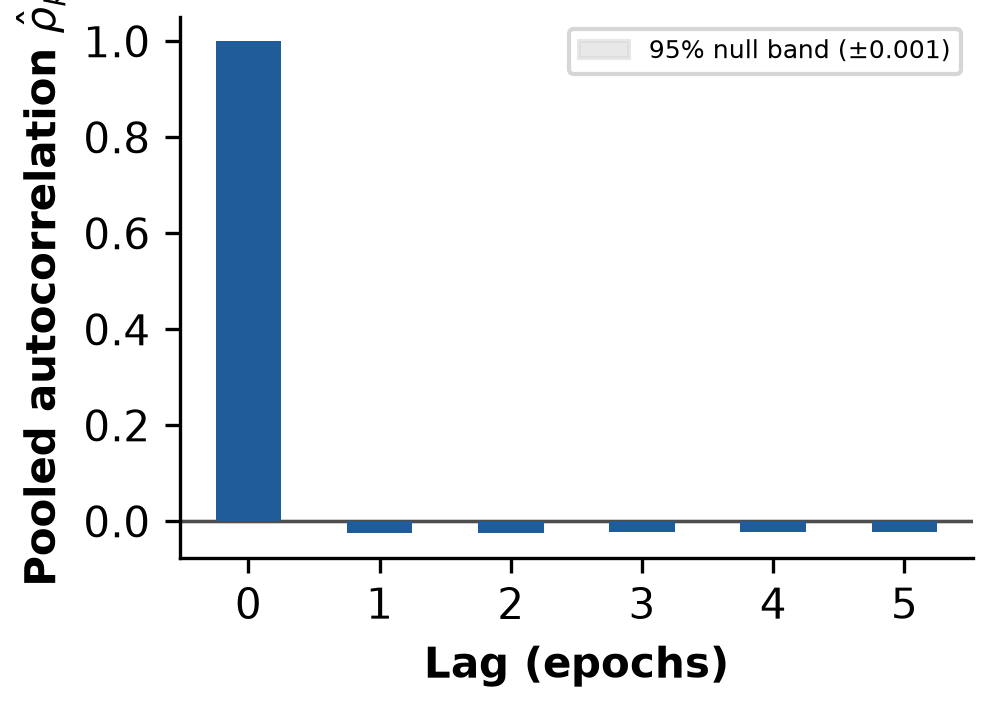}
\caption{Pooled lag-$k$ autocorrelation of detrended per-instance loss residuals, epochs 101--200, ResNet-18/CIFAR-10 seed 0 ($n{\approx}1.9$M pairs/lag). Shaded band: 95\% null interval for zero correlation.}
\label{fig:noiseautocorr}
\end{figure}

An earlier pass detrended with a width-5 moving average and measured $\hat\rho_1{=}{-}0.298,\hat\rho_2{=}{-}0.348$, which looked like strong evidence against \cref{ass:noise}; we verified by direct calculation that pure white noise under that same filter produces $\rho_1{=}{-}0.30,\rho_2{=}{-}0.35$ \emph{exactly} -- the entire signal was the detrending filter's own artifact. Redone with OLS linear detrending (whose own simulated white-noise artifact is $\hat\rho_1{\approx}{-}0.021,\hat\rho_2{\approx}{-}0.019$), the real data gives $\hat\rho_1{=}{-}0.025,\hat\rho_2{=}{-}0.023$ -- statistically distinguishable from zero at $n{\approx}1.9$M pairs but small and close to the detrending method's own bias, evidence \emph{for}, not against, \cref{ass:noise} once the method's bias is accounted for. Taking the measured values at face value regardless, the correction to \cref{thm:bv} at $N{=}2$ is a factor of $6.15$ against the assumed $6.0$ -- a $2.5\%$ correction, requiring $\delta$ to grow under $1.3\%$ to keep the same nominal false-exclusion rate. This does not materially change any result in \cref{sec:results}.

\section{Gradient-Norm-Floor Percentile Sensitivity}
\label{sec:sensitivity}

The calibration policy of \cref{sec:setup} fixes the percentile behind $\varepsilon_0$ (and hence the reference horizon for $K^*$) at the 20th throughout; we sweep $\{5,10,20,35,50,80\}$ on ResNet-18 and VGG-16 (primary recipe, seed 0), holding everything else fixed.

\begin{figure}[h]
\centering
\includegraphics[width=0.85\linewidth]{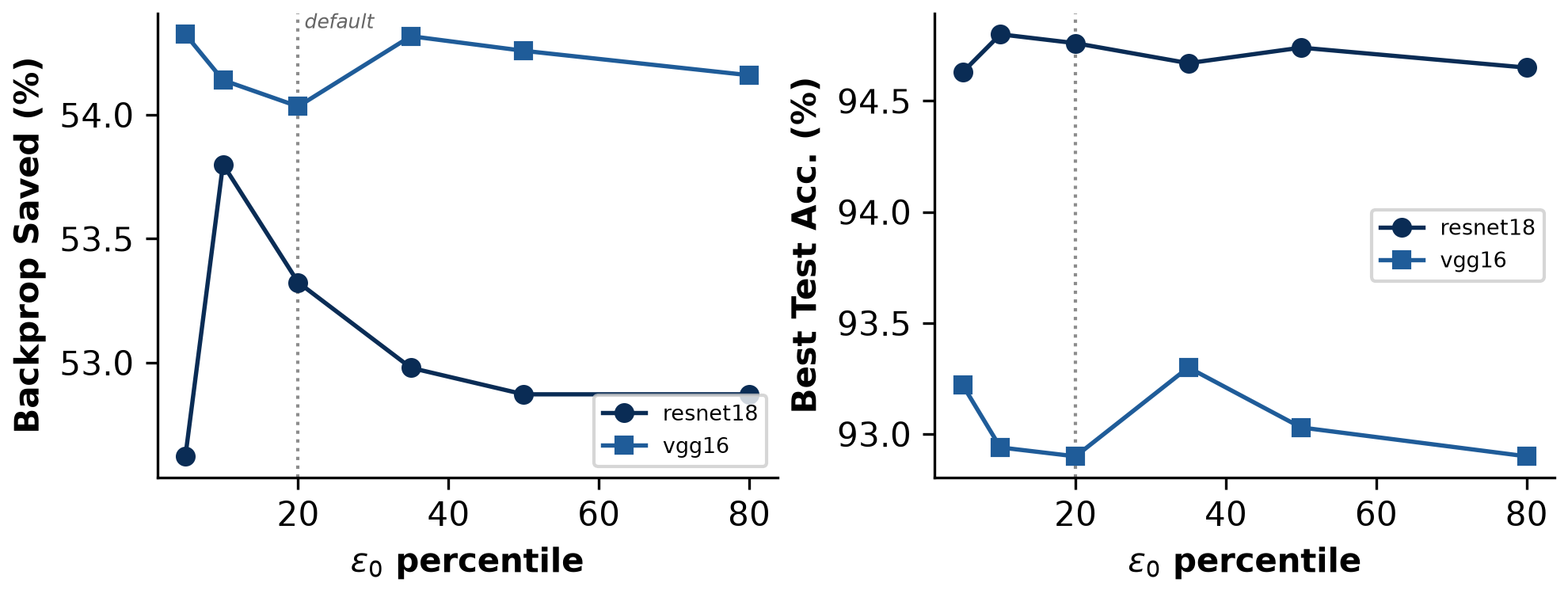}
\caption{Backprop saved (left) and best test accuracy (right) vs.\ gradient-norm-floor percentile, ResNet-18 and VGG-16, CIFAR-10, seed 0. Dotted line: this paper's fixed default (20th percentile).}
\label{fig:sensitivity}
\end{figure}

Both curves are essentially flat across a 16$\times$ range: ResNet-18 savings vary by $1.18$ points, accuracy by $0.17$; VGG-16 by $0.30$ and $0.40$ points respectively. Neither curve is monotonic, itself informative: there is no directional trend to correct for by tuning further. The gap between architectures at any fixed percentile is larger than the gap the percentile choice induces within either, consistent with the percentile being a fixed, low-sensitivity policy input rather than one needing per-architecture retuning within this range. This shows that the current value is not fragile; it does not show that a derived value would not be preferable in principle (\cref{rem:kmismatch}).

\section{Measured Wall-Clock Time, Peak Memory, and Estimated Energy}
\label{app:efficiency}

\begin{table}[h]
\caption{Measured efficiency, primary setting, single RTX 5090. Wall-clock: mean over 3 seeds from the real 200-epoch runs. Peak Mem/Util/Energy: short instrumented runs, seed 0; Energy is an estimate. Lower is better throughout.}
\label{tab:efficiency}
\centering
\small
\begin{tabular}{lcccc}
\toprule
Method & Wall-clk (min) & Peak Mem (MiB) & Util (\%) & Energy (Wh) \\
\midrule
Baseline & 15.66 & 1162 & 67.7 & 98.2 \\
IES (shipped) & 12.86 & 1216 & 62.8 & 78.1 \\
IES (literal) & 12.26 & 1184 & 59.0 & 70.1 \\
\win{TG-IES} & \win{$\mathbf{11.66}$} & \lose{1454} & \win{$\mathbf{48.9}$} & \win{$\mathbf{59.9}$} \\
\bottomrule
\end{tabular}
\end{table}

\begin{figure}[h]
\centering
\includegraphics[width=0.85\linewidth]{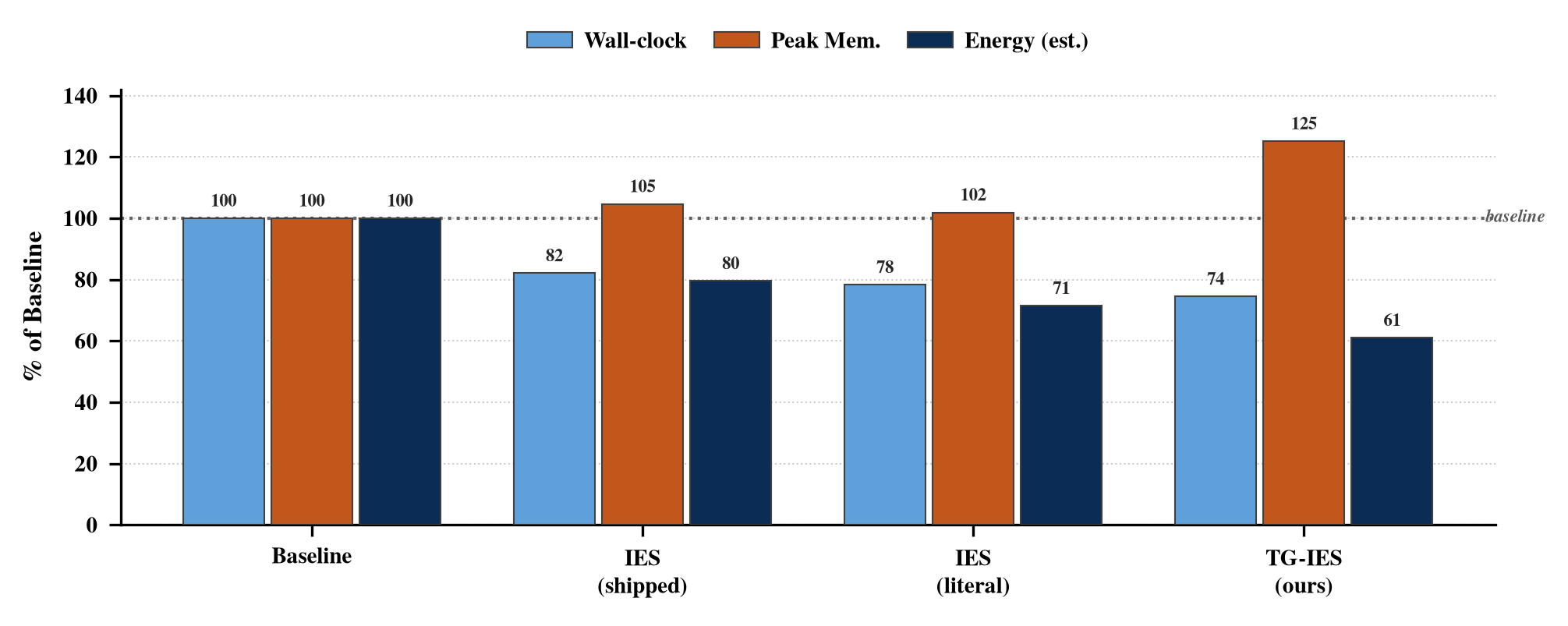}
\caption{Wall-clock time, peak GPU memory, and estimated energy, each as \% of no-removal baseline. Wall-clock and energy fall furthest for TG-IES; peak memory is the one metric that rises above baseline, only for TG-IES.}
\label{fig:efficiency}
\end{figure}

Wall-clock is the per-epoch time already logged by every 200-epoch run; peak memory, mean utilization, and mean power were collected separately via out-of-process \texttt{nvidia-smi} polling on fresh 8-epoch instrumented repetitions; estimated energy multiplies each method's measured mean power by its real 200-epoch wall-clock. TG-IES is fastest and most estimated-energy-efficient ($25.5\%$ faster, $\sim\!39\%$ lower energy than baseline) but has the \emph{highest} peak memory ($1454$ MiB, $\sim\!25\%$ above baseline), attributable specifically to its one-time recalibration pass: estimating the gradient-norm-floor percentile forwards a subsample in batches of $256$ (four times the training batch size), and PyTorch's caching allocator does not release the resulting larger reservation for the rest of the run. This is a genuine memory/compute trade-off specific to TG-IES's calibration mechanism.

\section{Additional Figures for Finding 2 and the Mastery-to-Gradient Bridge}
\label{app:extra-figs}

\begin{figure}[h]
\centering
\includegraphics[width=0.48\linewidth]{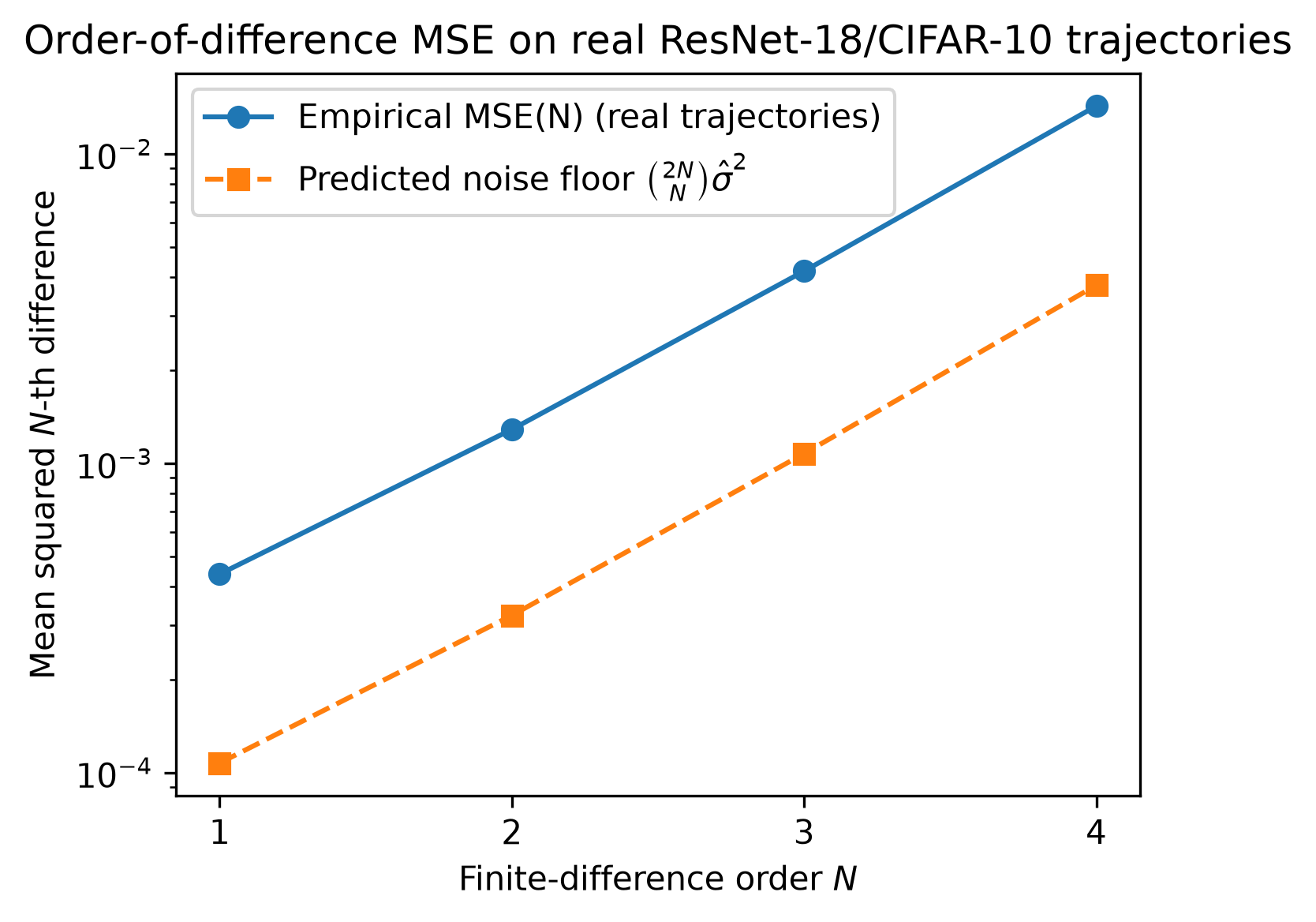}
\hfill
\includegraphics[width=0.48\linewidth]{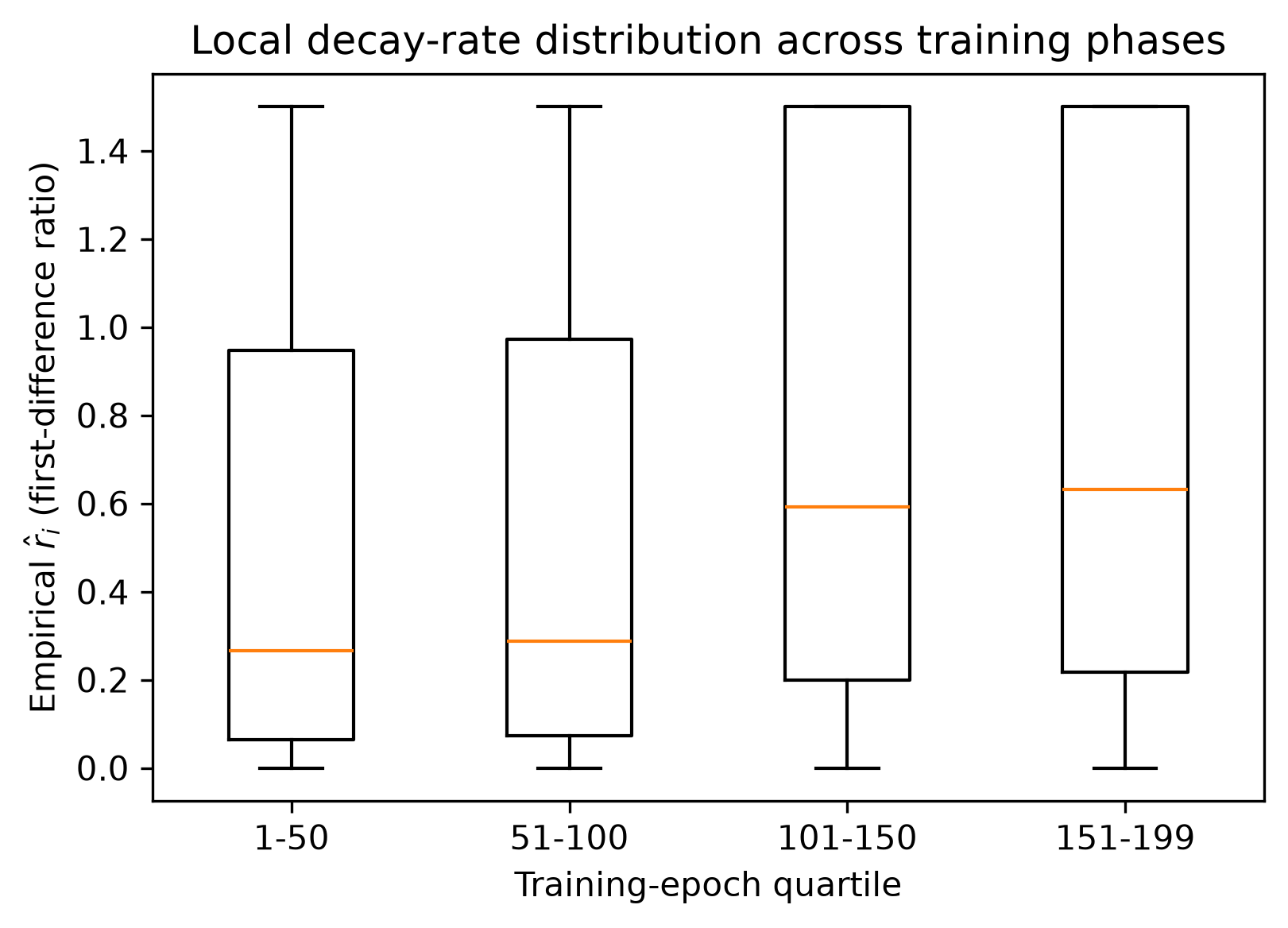}
\caption{Left: empirical MSE$(N)$ vs.\ predicted noise floor, naive population-average test (all instances, epochs 101--200). This pooled test does not isolate \cref{prop:mse}'s SNR band and finds $N{=}1$ lowest; the band-restricted test (main text), which evaluates each instance at the epoch its signal first drops below the calibrated $\delta{=}0.01597$, finds $N{=}2$ lowest ($7.5\times10^{-5}$ pooled over $46{,}779$ instances), as predicted. Right: empirical $\hat r_i$ pooled per training-epoch quartile; rate rises and spread narrows in the second half of training (mean $\hat r$ from $0.53$ to $0.76$; CV from $1.05$ to $0.75$), so \cref{ass:geometric} fits late training better.}
\label{fig:orderdiff}
\label{fig:phasequartile}
\end{figure}

\begin{figure}[h]
\centering
\includegraphics[width=0.95\linewidth]{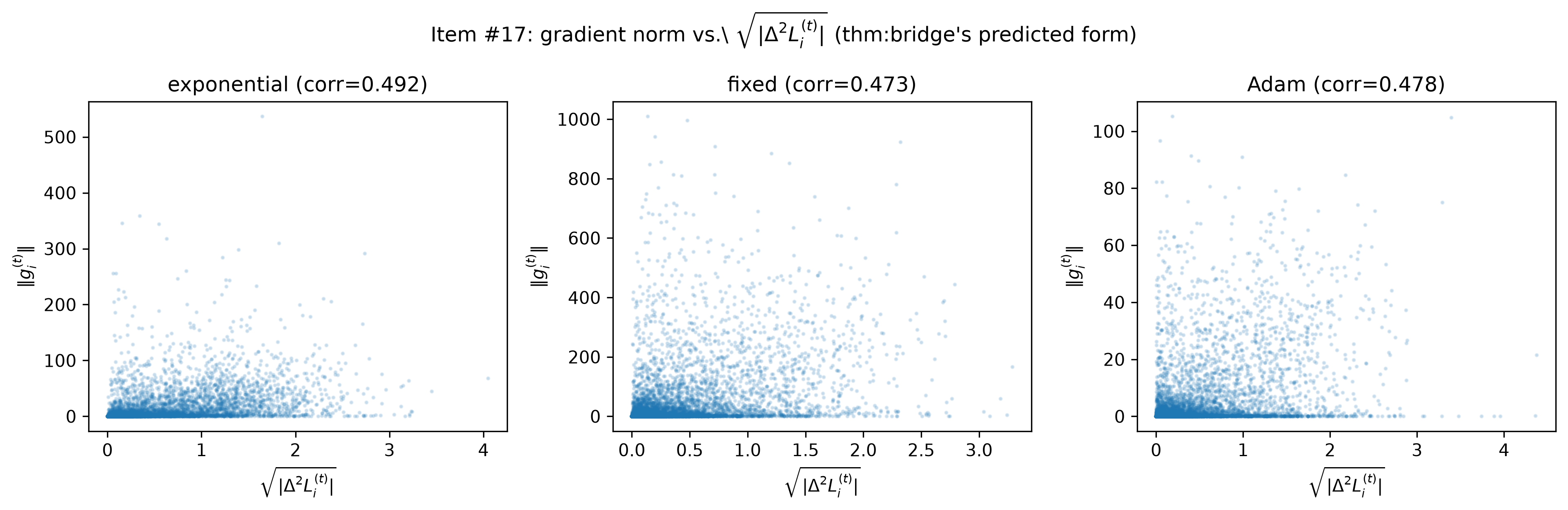}
\caption{Per-instance gradient norm $\|g_i^{(t)}\|$ vs.\ $\sqrt{|\Delta^2 L_i^{(t)}|}$, all three LR schedules, primary architecture/dataset (300-instance subset). \Cref{thm:bridge} predicts a positive, roughly linear relationship; correlations shown per panel.}
\label{fig:bridgenormcheck}
\end{figure}

\section{ViT-Tiny: Full Results and Caveats}
\label{app:vit}

Every architecture in the main matrix is convolutional; here we test whether TG-IES's calibration -- derived under no architectural assumption beyond a per-instance loss and gradient -- transfers to a Vision Transformer \citep{dosovitskiy2021image}, trained from scratch at CIFAR scale with the same four fixed calibration constants as every convolutional cell. We use a CIFAR-native ViT-Tiny: patch-4 on the $32{\times}32$ input (64 tokens plus CLS), 6 pre-norm blocks, 192-dim, 3 heads, $2.69$M parameters -- not the ImageNet-oriented patch-16 configuration, whose tokenization would yield only 4 tokens at this resolution. AdamW (lr $10^{-3}$, weight decay $0.05$), 10-epoch linear warmup then cosine decay over 200 epochs; we validated the baseline alone before running the full comparison.

\begin{table}[h]
\caption{ViT-Tiny (CIFAR-native, patch-4), CIFAR-10 and CIFAR-100, AdamW warmup-cosine, mean $\pm$ std, 3 seeds.}
\label{tab:vit10}
\label{tab:vit100}
\centering
\small
\begin{tabular}{lcccc}
\toprule
& \multicolumn{2}{c}{CIFAR-10} & \multicolumn{2}{c}{CIFAR-100} \\
Method & Acc.\ (\%) & Saved (\%) & Acc.\ (\%) & Saved (\%) \\
\midrule
Baseline & $84.07\pm0.35$ & $0.03\pm0.00$ & $59.25\pm0.24$ & $0.03\pm0.00$ \\
IES (shipped) & $83.92\pm0.18$ & $13.13\pm0.21$ & $59.19\pm0.22$ & $8.38\pm0.22$ \\
IES (literal) & $83.86\pm0.48$ & $16.68\pm0.32$ & $58.77\pm0.24$ & $12.51\pm0.21$ \\
\win{TG-IES (ours)} & \lose{$83.71\pm0.47$} & \win{$\mathbf{25.52\pm0.26}$} & \lose{$58.56\pm0.47$} & \win{$\mathbf{26.28\pm0.14}$} \\
\bottomrule
\end{tabular}
\end{table}

\textbf{Four findings.} The savings advantage transfers and grows with difficulty: $1.53\times$ IES (literal) on CIFAR-10, $2.10\times$ on CIFAR-100, tight seed variance on savings specifically. Noise-adaptive threshold calibration, not the fixed-$\delta$ methods' hand-tuned constant, is the mechanism: $\delta$ rescales $186\times$ (CIFAR-10) and $333\times$ (CIFAR-100) with no constant retuned by hand. Absolute savings are much lower than CNNs ($\sim\!26\%$ vs.\ $\sim\!54.5\%$ DenseNet-121). There is a small accuracy cost on transformers unlike CNNs ($0.36$--$0.69$ points, widest seed spread of any method), suggesting instance removal destabilizes a data-starved, from-scratch ViT more than a CNN.

\textbf{Caveats.} $K^*$ was pinned at the configured ceiling ($15$) in every ViT run: \cref{cor:kstar}'s closed form asked for more patience than allowed, so these numbers reflect the theory clamped by configuration; lifting the cap is untested. Savings arrive late (exclusions ramp steeply only in the second half of training). $\delta$ recalibrates once, at epoch 30 as in every CNN cell; whether this timing is also right for a transformer's different loss dynamics is untested. Absolute accuracy is well below the CNN cells, expected for a from-scratch transformer on 50k images with crop-and-flip augmentation only.

\section{\texorpdfstring{$K^*$}{K*}-Sweep: Full Trajectory Figure and Discussion}
\label{app:kstartraj}

\begin{table}[h]
\caption{$K^*$-sweep, primary cell, mean $\pm$ std, 3 seeds, $\delta$ noise-calibrated. Matched baseline: $95.08\pm0.14\%$ acc.}
\label{tab:kstarsweep-full}
\centering
\small
\begin{tabular}{ccc}
\toprule
$K^*$ (fixed) & Acc.\ (\%) & Saved (\%) \\
\midrule
1 & $94.55\pm0.16$ & $\mathbf{66.98\pm0.62}$ \\
2 & $94.56\pm0.08$ & $65.75\pm0.27$ \\
4 & $94.62\pm0.10$ & $63.24\pm0.44$ \\
8 & $94.80\pm0.09$ & $58.85\pm0.57$ \\
15 & $94.80\pm0.09$ & $53.38\pm0.58$ \\
30 & $\mathbf{94.87\pm0.09}$ & $44.13\pm0.40$ \\
\bottomrule
\end{tabular}
\end{table}

\begin{figure}[h]
\centering
\includegraphics[width=0.95\linewidth]{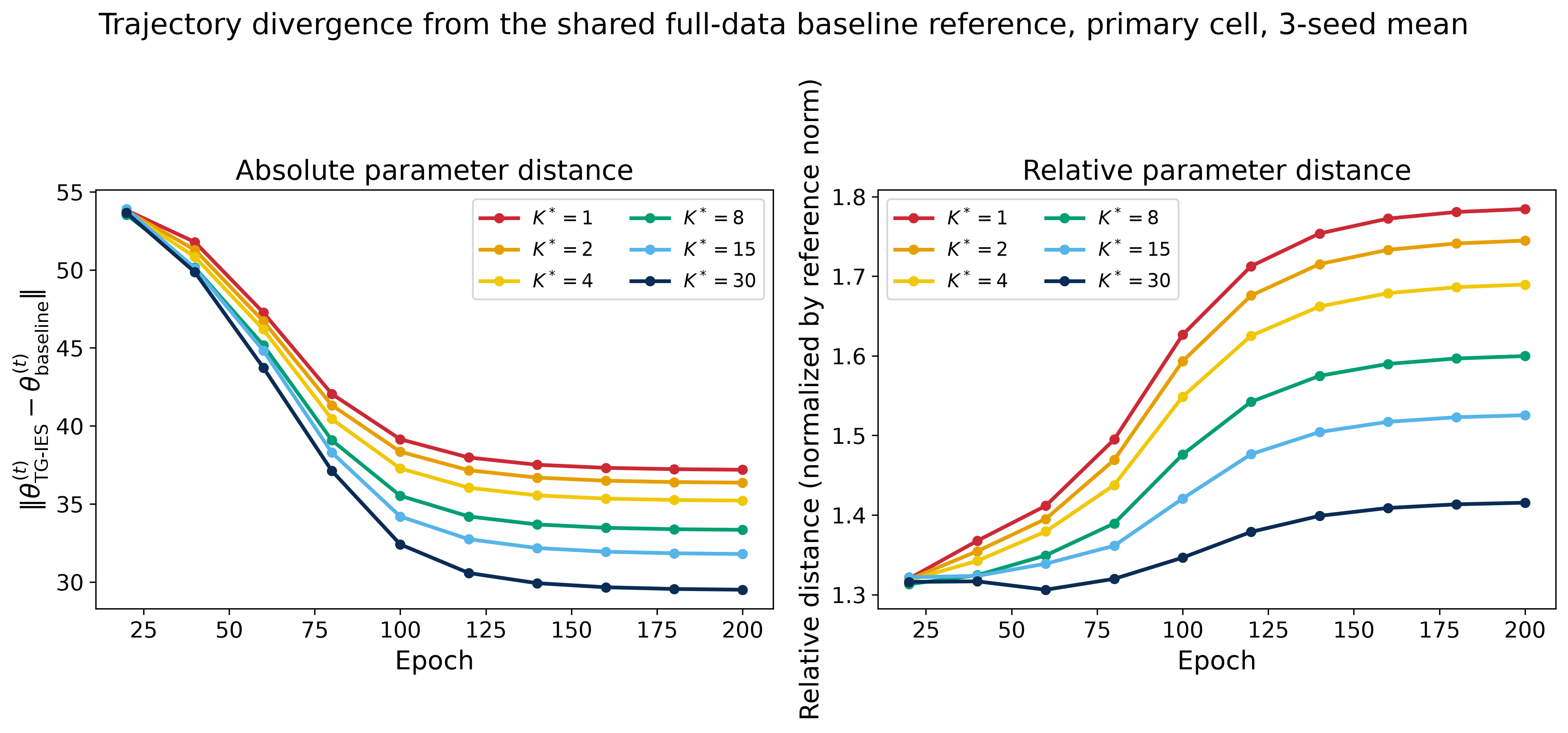}
\caption{Parameter-space distance between TG-IES's actual trajectory and the shared-seed baseline reference vs.\ epoch, one curve per fixed $K^*$, 3-seed mean. Left: absolute distance. Right: normalized by the baseline's own parameter norm at that epoch.}
\label{fig:kstartraj}
\end{figure}

Absolute parameter distance from baseline \emph{decreases} over training for every $K^*$ tested, from $\approx\!54$ at epoch $20$ to $29$--$37$ by epoch $200$, plateauing over the last third of training. At every snapshot epoch, larger $K^*$ gives a smaller distance, with no crossings. This is the direction suggested by \cref{thm:actual}: longer patience delays exclusion and keeps the passive set, and plausibly the residual $\omega_t$, smaller. The \emph{normalized} distance (right panel) instead rises and plateaus higher, because the reference parameter norm itself shrinks faster (weight decay) than the absolute gap does. We report both panels, since this experiment does not settle which normalization is the right one for assessing multi-step deviation. We read the plateau as empirical behavior consistent with a bounded aggregate residual. It is not a validation of \cref{thm:actual}, whose residual bound is not measured here, and the plateau is also not predicted by that theorem's smoothness-only bound, which grows with the horizon (\cref{rem:contraction-direction}).

\section{Reactivation Probe: Does the Excluded Set Actually Satisfy the Gradient-Floor Hypothesis?}
\label{app:reactprobe}

The single-removal certificate (\cref{thm:safe}), the reference horizon used to initialize $K^*$ (\cref{cor:kstar}), and the per-instance multi-step bound (\cref{prop:multistep}) all rest on a per-instance hypothesis: every currently excluded instance $i$ satisfies $\|g_i^{(t)}\|<\varepsilon_0$, where $\varepsilon_0$ is the cheap last-layer proxy percentile estimated once at epoch 30 (\cref{app:practical}). The other checks in this paper (\cref{sec:bridge}'s empirical check, \cref{app:extra-figs}) use the same cheap proxy or a correlational test, not the \emph{true} full-model gradient norm on the actual excluded set during an otherwise unmodified run. We test the hypothesis directly: \texttt{-{}-reactivation\_probe\_every 20} adds, every 20 epochs, an individual batch-size-1 forward+backward pass on a random subsample of $300$ currently-passive instances (noise-calibrated $\delta$, primary cell, 3 seeds, tagged \texttt{\_reactprobe}), computing each one's true $\|g_i^{(t)}\|$ and checking it against the run's own calibrated $\varepsilon_0$. This instrumentation is purely diagnostic: it does not feed back into the algorithm's actual removal or reactivation decisions, so the probed runs' accuracy and savings ($94.67\pm0.11\%$ acc., $53.66\pm0.10\%$ saved) match the unprobed primary-cell TG-IES numbers (\cref{tab:primary}) closely, confirming the probe itself does not perturb training.

\begin{table}[h]
\caption{Reactivation probe, primary cell, 3-seed mean of \texttt{frac\_exceed\_eps0} (fraction of a 300-instance random subsample of the passive set whose true gradient norm exceeds the run's calibrated $\varepsilon_0$) at each probed epoch.}
\label{tab:reactprobe}
\centering
\small
\begin{tabular}{cccccccccc}
\toprule
Epoch & 40 & 60 & 80 & 100 & 120 & 140 & 160 & 180 & 200 \\
\midrule
Frac.\ $>\varepsilon_0$ & $0.52$ & $0.48$ & $0.39$ & $0.46$ & $0.48$ & $0.46$ & $0.47$ & $0.50$ & $0.50$ \\
\bottomrule
\end{tabular}
\end{table}

\textbf{Finding: roughly half the excluded set exceeds the calibrated floor at every check, with no decreasing trend.} Pooling all 9 probed epochs across 3 seeds ($27$ (epoch, seed) points, $300$ probed instances each), the mean fraction of the passive set exceeding $\varepsilon_0$ is $47.4\%$, ranging $34$--$59\%$ epoch-to-epoch with no consistent decrease over training (\cref{tab:reactprobe}) -- if anything, the fraction is lowest at epoch 80 and drifts back up to $\approx\!50\%$ by epoch 200, the opposite of what a hypothesis that becomes \emph{more} valid as training progresses would predict. This contradicts a literal reading of the per-instance hypothesis: at any given check, roughly half of the instances the algorithm currently excludes have a true gradient norm above the floor that the per-instance analysis assumes. The picture is not uniformly bad, however: each seed's \emph{median} probed gradient norm ($0.025$--$0.029$ across seeds and epochs) sits close to that seed's own $\varepsilon_0$ ($0.027$--$0.030$), while the \emph{mean} probed gradient norm is $2$--$10\times$ higher ($0.07$--$0.29$) -- a heavy right tail, not a uniformly shifted distribution. So the calibrated $\varepsilon_0$ is a reasonable estimate of a \emph{typical} (median) excluded instance's gradient floor, but a systematically loose one for the population as a whole, since roughly half the excluded set sits above it at any moment.

\textbf{Why this motivates the aggregate formulation.} On the same primary cell, the trajectory-divergence measurement of \cref{sec:kstartraj} finds that the absolute parameter gap from the full-data baseline \emph{plateaus} rather than growing, although the per-instance violation rate never approaches zero. The inexact-gradient view of \cref{sec:inexact} locates the quantity that matters for this comparison: by \cref{thm:actual}, multi-step deviation depends on the passive set only through the aggregate residual $\omega_t=\frac1N\sum_{i\in P_t}g_i$, not through any individual $\|g_i\|$. The per-instance floor is thus a sufficient but unnecessarily strong way to control $\omega_t$ (\cref{prop:multistep}), and this probe shows that it is too strong to be treated as an empirical property of the passive set. The probe does \emph{not}, however, measure $\omega_t$, so it does not establish the aggregate condition required by \cref{thm:actual}. We therefore keep that condition as an explicit assumption (\cref{sec:limitations}) rather than claiming it has been verified, and we treat the plateau as behavior consistent with a bounded aggregate residual, not as proof of one.

\section{Extended Future-Work List}
\label{app:limitations-extended}

(1) \emph{Scale.} All experiments are CIFAR-scale; TinyImageNet or ImageNet-100 would test whether calibration (in particular the gradient-norm-floor proxy behind $K^*$) transfers to higher resolution and more classes. (2) \emph{Architecture family, partly closed.} \Cref{app:vit} adds one CIFAR-native ViT-Tiny; a larger transformer (ViT-Small, DeiT-Tiny) or one at a resolution where patch-16 tokenization is meaningful remains untried, as does raising \texttt{-{}-max\_kstar}. (3) \emph{Optimizer/schedule sensitivity.} \Cref{sec:settings,app:adam} show TG-IES trailing simpler IES variants on settings without the exponential schedule; \cref{sec:lrmechanism}'s candidate mechanism is only partly confirmed by direct instrumentation, and a controlled ablation calibrating $\delta,K^*$ under a matched noise/gradient-norm profile across schedules remains unrun. The bridge's qualitative direction ($\|g_i^{(t)}\|$ vs.\ $\sqrt{|\Delta^2L_i^{(t)}|}$) is now confirmed across all three schedules; only the \emph{exact} bound $\varepsilon_i(\delta)$, needing per-instance $r_i,\Lambda_i$, remains open. (4) RHO-LOSS \citep{mindermann2022prioritized} and RL-style dynamic pruning \citep{zhu2026dppo} are not compared. (5) \emph{The aggregate residual.} \Cref{thm:actual,thm:descent} are conditional on a bound for $\omega_t$, which the current implementation neither measures nor controls. Measuring $\|\omega_t\|$ during training, bounding it from observable quantities, and turning such a bound into a residual-aware reactivation rule would each require instrumentation beyond the present study. A contraction-based version of \cref{thm:actual} that removes the $(1+\beta\eta)^H$ factor (\cref{rem:contraction-direction}) is the corresponding theoretical step. We view these as the most important remaining gaps, ahead of (1)--(4).

\section{Additional Experimental Details}
\label{app:experiment}

\textbf{A reproducibility correction.} \label{app:repro}An earlier version reported TG-IES collapsing to $1.8\%$ savings on VGG-16. We traced this to a stale configuration that repeated the recalibration step. Re-run with the one-time recalibration used everywhere else, VGG-16 gives the largest saving of any method on that setting ($54.43\%$, \cref{tab:overview}). The fixed-LR ResNet-18 gap to IES (literal) ($29.27\%$ vs.\ $49.42\%$) is genuine and persists after re-running. The failure mode that the stale configuration had suggested, a patience rule incompatible with the threshold, is real, and it is demonstrated by the controlled ablation of \cref{sec:ablation} rather than by this corrected data point.

\textbf{Full per-run results.} \Cref{tab:fullresults} lists every IES-family run underlying the main tables and \cref{app:adam,app:vit} (96 runs total). Two rows (ResNet-18/F and VGG-16/E, TG-IES, seed 0) supersede values from an earlier version, traced to the reproducibility issue above; stale originals are retained in the released code's \texttt{results/} directory with a \texttt{\_STALE\_preverify} suffix for audit purposes.

\begingroup
\small
\begin{longtable}{llccc}
\caption{All 96 IES-family runs. E = SGD exponential LR, F = SGD fixed LR, C100 = CIFAR-100, ViT/C10 and ViT/C100 = ViT-Tiny on CIFAR-10/CIFAR-100. Color on TG-IES rows compares against the other methods for that same seed.}
\label{tab:fullresults} \\
\toprule
Setting & Method & Seed & Acc.\ (\%) & Saved (\%) \\
\midrule
\endfirsthead
\multicolumn{5}{l}{\tablename~\thetable{} -- continued from previous page} \\
\toprule
Setting & Method & Seed & Acc.\ (\%) & Saved (\%) \\
\midrule
\endhead
\midrule
\multicolumn{5}{r}{continued on next page} \\
\endfoot
\bottomrule
\endlastfoot
ResNet-18/E & Baseline & 0 & 94.96 & 0.03 \\
ResNet-18/E & Baseline & 1 & 95.18 & 0.03 \\
ResNet-18/E & Baseline & 2 & 94.95 & 0.03 \\
ResNet-18/E & IES (shipped) & 0 & 94.61 & 45.36 \\
ResNet-18/E & IES (shipped) & 1 & 94.74 & 46.09 \\
ResNet-18/E & IES (shipped) & 2 & 94.79 & 45.53 \\
ResNet-18/E & IES (literal) & 0 & 94.92 & 48.59 \\
ResNet-18/E & IES (literal) & 1 & 94.66 & 49.00 \\
ResNet-18/E & IES (literal) & 2 & 94.97 & 49.47 \\
ResNet-18/E & TG-IES & 0 & \lose{94.61} & \win{53.48} \\
ResNet-18/E & TG-IES & 1 & \lose{95.16} & \win{53.44} \\
ResNet-18/E & TG-IES & 2 & \lose{94.77} & \win{53.98} \\
ResNet-18/F & Baseline & 0 & 92.28 & 0.03 \\
ResNet-18/F & Baseline & 1 & 92.21 & 0.03 \\
ResNet-18/F & Baseline & 2 & 92.13 & 0.03 \\
ResNet-18/F & IES (shipped) & 0 & 91.78 & 44.71 \\
ResNet-18/F & IES (shipped) & 1 & 91.69 & 44.55 \\
ResNet-18/F & IES (shipped) & 2 & 91.61 & 44.23 \\
ResNet-18/F & IES (literal) & 0 & 91.77 & 49.60 \\
ResNet-18/F & IES (literal) & 1 & 91.66 & 49.38 \\
ResNet-18/F & IES (literal) & 2 & 91.87 & 49.27 \\
ResNet-18/F & TG-IES & 0 & \lose{91.71} & \lose{29.57} \\
ResNet-18/F & TG-IES & 1 & \lose{91.96} & \lose{29.27} \\
ResNet-18/F & TG-IES & 2 & \lose{92.07} & \lose{28.98} \\
VGG-16/E & Baseline & 0 & 93.25 & 0.03 \\
VGG-16/E & Baseline & 1 & 93.21 & 0.03 \\
VGG-16/E & Baseline & 2 & 93.44 & 0.03 \\
VGG-16/E & IES (shipped) & 0 & 93.33 & 38.67 \\
VGG-16/E & IES (shipped) & 1 & 93.27 & 37.46 \\
VGG-16/E & IES (shipped) & 2 & 93.34 & 36.99 \\
VGG-16/E & IES (literal) & 0 & 93.04 & 41.85 \\
VGG-16/E & IES (literal) & 1 & 93.35 & 41.45 \\
VGG-16/E & IES (literal) & 2 & 93.07 & 39.74 \\
VGG-16/E & TG-IES & 0 & \lose{92.96} & \win{54.36} \\
VGG-16/E & TG-IES & 1 & \lose{93.15} & \win{54.56} \\
VGG-16/E & TG-IES & 2 & \lose{93.23} & \win{54.37} \\
DenseNet-121/E & Baseline & 0 & 95.37 & 0.03 \\
DenseNet-121/E & Baseline & 1 & 95.50 & 0.03 \\
DenseNet-121/E & Baseline & 2 & 95.22 & 0.03 \\
DenseNet-121/E & IES (shipped) & 0 & 95.37 & 48.19 \\
DenseNet-121/E & IES (shipped) & 1 & 95.47 & 49.79 \\
DenseNet-121/E & IES (shipped) & 2 & 95.25 & 47.83 \\
DenseNet-121/E & IES (literal) & 0 & 95.26 & 51.96 \\
DenseNet-121/E & IES (literal) & 1 & 95.24 & 52.27 \\
DenseNet-121/E & IES (literal) & 2 & 95.23 & 50.59 \\
DenseNet-121/E & TG-IES & 0 & \lose{95.24} & \win{54.73} \\
DenseNet-121/E & TG-IES & 1 & \lose{95.22} & \win{55.06} \\
DenseNet-121/E & TG-IES & 2 & \win{95.38} & \win{53.79} \\
ResNet-18/Adam & Baseline & 0 & 92.98 & 0.03 \\
ResNet-18/Adam & Baseline & 1 & 92.88 & 0.03 \\
ResNet-18/Adam & Baseline & 2 & 92.86 & 0.03 \\
ResNet-18/Adam & IES (shipped) & 0 & 92.67 & 56.04 \\
ResNet-18/Adam & IES (shipped) & 1 & 92.50 & 55.80 \\
ResNet-18/Adam & IES (shipped) & 2 & 92.66 & 56.23 \\
ResNet-18/Adam & IES (literal) & 0 & 92.49 & 59.80 \\
ResNet-18/Adam & IES (literal) & 1 & 92.47 & 60.39 \\
ResNet-18/Adam & IES (literal) & 2 & 92.30 & 60.10 \\
ResNet-18/Adam & TG-IES & 0 & \lose{92.33} & \lose{30.86} \\
ResNet-18/Adam & TG-IES & 1 & \lose{92.49} & \lose{29.64} \\
ResNet-18/Adam & TG-IES & 2 & \lose{92.51} & \lose{30.04} \\
ResNet-18/C100 & Baseline & 0 & 76.46 & 0.03 \\
ResNet-18/C100 & Baseline & 1 & 77.27 & 0.03 \\
ResNet-18/C100 & Baseline & 2 & 76.91 & 0.03 \\
ResNet-18/C100 & IES (shipped) & 0 & 76.64 & 11.74 \\
ResNet-18/C100 & IES (shipped) & 1 & 76.69 & 11.99 \\
ResNet-18/C100 & IES (shipped) & 2 & 77.29 & 12.19 \\
ResNet-18/C100 & IES (literal) & 0 & 76.33 & 18.18 \\
ResNet-18/C100 & IES (literal) & 1 & 76.92 & 18.32 \\
ResNet-18/C100 & IES (literal) & 2 & 77.09 & 18.98 \\
ResNet-18/C100 & TG-IES & 0 & \lose{75.97} & \win{57.81} \\
ResNet-18/C100 & TG-IES & 1 & \lose{75.96} & \win{57.73} \\
ResNet-18/C100 & TG-IES & 2 & \lose{76.11} & \win{57.75} \\
ViT-Tiny/C10 & Baseline & 0 & 84.25 & 0.03 \\
ViT-Tiny/C10 & Baseline & 1 & 84.30 & 0.03 \\
ViT-Tiny/C10 & Baseline & 2 & 83.67 & 0.03 \\
ViT-Tiny/C10 & IES (shipped) & 0 & 84.13 & 13.37 \\
ViT-Tiny/C10 & IES (shipped) & 1 & 83.83 & 12.97 \\
ViT-Tiny/C10 & IES (shipped) & 2 & 83.81 & 13.06 \\
ViT-Tiny/C10 & IES (literal) & 0 & 84.32 & 17.05 \\
ViT-Tiny/C10 & IES (literal) & 1 & 83.90 & 16.48 \\
ViT-Tiny/C10 & IES (literal) & 2 & 83.37 & 16.52 \\
ViT-Tiny/C10 & TG-IES & 0 & \lose{84.26} & \win{25.80} \\
ViT-Tiny/C10 & TG-IES & 1 & \lose{83.47} & \win{25.30} \\
ViT-Tiny/C10 & TG-IES & 2 & \lose{83.41} & \win{25.45} \\
ViT-Tiny/C100 & Baseline & 0 & 59.19 & 0.03 \\
ViT-Tiny/C100 & Baseline & 1 & 59.51 & 0.03 \\
ViT-Tiny/C100 & Baseline & 2 & 59.05 & 0.03 \\
ViT-Tiny/C100 & IES (shipped) & 0 & 58.96 & 8.20 \\
ViT-Tiny/C100 & IES (shipped) & 1 & 59.40 & 8.62 \\
ViT-Tiny/C100 & IES (shipped) & 2 & 59.20 & 8.33 \\
ViT-Tiny/C100 & IES (literal) & 0 & 58.54 & 12.28 \\
ViT-Tiny/C100 & IES (literal) & 1 & 59.02 & 12.70 \\
ViT-Tiny/C100 & IES (literal) & 2 & 58.75 & 12.55 \\
ViT-Tiny/C100 & TG-IES & 0 & \lose{58.97} & \win{26.23} \\
ViT-Tiny/C100 & TG-IES & 1 & \lose{58.05} & \win{26.43} \\
ViT-Tiny/C100 & TG-IES & 2 & \lose{58.66} & \win{26.17} \\
\end{longtable}
\endgroup

\textbf{Hardware and reproducibility.} All runs used a single NVIDIA RTX 5090 GPU; \cref{app:efficiency} reports precise wall-clock, memory, utilization, and estimated energy. All four methods are implemented in a single training entry point taking \texttt{-{}-method \{baseline,ies\_shipped,ies\_alg1,tgies\}} so data pipeline, model, and optimizer code are shared exactly across methods; a single script drives the full experimental matrix and supports resuming any interrupted run.

\section{Statistical Analysis of Seed-Level Comparisons}
\label{app:stats}

Every $\pm$ figure elsewhere in this paper is a sample standard deviation over 3 seeds, not a significance test. This section reports one directly. \texttt{setup\_seed(seed)} is called once per run, before model initialization and data loading, so all four methods sharing a seed share the same initial weights and data-shuffling order -- a genuine paired design, analyzed via paired per-seed differences rather than treating (method, seed) runs as independent. Eight settings have 3 seeds: the primary cell, both secondary CNN cells, Adam, CIFAR-100, DenseNet-121, and both ViT-Tiny cells; the three LR-sweep cells of \cref{sec:lrmechanism} remain single-seed by design.

For each 3-seed setting we compare TG-IES against each other method on best test accuracy, backprop-saved \%, and FLOP-saved \%, reporting: the paired mean difference and its sample standard deviation; a $95\%$ CI via the exact $t$-distribution at $2$ degrees of freedom ($t_{0.975,2}=4.303$); Cohen's $d_z$ (paired effect size); and an exact two-sided sign test (floor $p=2/8=0.25$ at $n=3$, reached whenever all seeds agree in sign, regardless of effect size).

\begingroup
\small
\begin{longtable}{llrrrr}
\caption{Paired per-seed statistics, TG-IES vs.\ each other method, every 3-seed setting ($n=3$). CI: $95\%$, $t(\mathrm{df}{=}2)$. $d_z$: Cohen's paired effect size. $p_{\mathrm{sign}}$: exact two-sided sign test (floor $0.25$ at $n=3$).}
\label{tab:stats} \\
\toprule
Metric & Comparison (TG-IES $-$) & Mean $\Delta$ & $95\%$ CI & $d_z$ & $p_{\mathrm{sign}}$ \\
\midrule
\endfirsthead
\multicolumn{6}{l}{\tablename~\thetable{} -- continued from previous page} \\
\toprule
Metric & Comparison (TG-IES $-$) & Mean $\Delta$ & $95\%$ CI & $d_z$ & $p_{\mathrm{sign}}$ \\
\midrule
\endhead
\midrule
\multicolumn{6}{r}{continued on next page} \\
\endfoot
\bottomrule
\endlastfoot
\multicolumn{6}{l}{\textbf{Primary: ResNet-18, CIFAR-10, SGD (exponential LR)}} \\
\midrule
Acc.\ (\%) & Baseline & \lose{$-0.18$} & $[-0.59,\,0.23]$ & $-1.11$ & $0.250$ \\
Acc.\ (\%) & IES (shipped) & \win{$0.13$} & $[-0.48,\,0.75]$ & $0.54$ & $1.000$ \\
Acc.\ (\%) & IES (literal) & \lose{$0.00$} & $[-1.09,\,1.09]$ & $-0.01$ & $1.000$ \\
Backprop Saved (\%) & Baseline & \win{$53.60$} & $[52.86,\,54.34]$ & $179.2$ & $0.250$ \\
Backprop Saved (\%) & IES (shipped) & \win{$7.97$} & $[6.57,\,9.37]$ & $14.2$ & $0.250$ \\
Backprop Saved (\%) & IES (literal) & \win{$4.61$} & $[4.00,\,5.22]$ & $18.8$ & $0.250$ \\
FLOP Saved (\%) & Baseline & \win{$35.71$} & $[35.22,\,36.21]$ & $179.2$ & $0.250$ \\
FLOP Saved (\%) & IES (shipped) & \win{$5.31$} & $[4.38,\,6.24]$ & $14.2$ & $0.250$ \\
FLOP Saved (\%) & IES (literal) & \win{$3.07$} & $[2.67,\,3.48]$ & $18.8$ & $0.250$ \\
\midrule
\multicolumn{6}{l}{\textbf{Secondary: ResNet-18, CIFAR-10, SGD (fixed LR)}} \\
\midrule
Acc.\ (\%) & Baseline & \lose{$-0.29$} & $[-0.93,\,0.35]$ & $-1.14$ & $0.250$ \\
Acc.\ (\%) & IES (shipped) & \win{$0.22$} & $[-0.45,\,0.89]$ & $0.82$ & $1.000$ \\
Acc.\ (\%) & IES (literal) & \win{$0.15$} & $[-0.31,\,0.61]$ & $0.79$ & $1.000$ \\
Backprop Saved (\%) & Baseline & \win{$29.24$} & $[28.51,\,29.98]$ & $98.8$ & $0.250$ \\
Backprop Saved (\%) & IES (shipped) & \lose{$-15.23$} & $[-15.42,\,-15.03]$ & $-194.1$ & $0.250$ \\
Backprop Saved (\%) & IES (literal) & \lose{$-20.14$} & $[-20.48,\,-19.80]$ & $-148.1$ & $0.250$ \\
FLOP Saved (\%) & Baseline & \win{$19.49$} & $[18.99,\,19.97]$ & $98.8$ & $0.250$ \\
FLOP Saved (\%) & IES (shipped) & \lose{$-10.14$} & $[-10.27,\,-10.02]$ & $-194.1$ & $0.250$ \\
FLOP Saved (\%) & IES (literal) & \lose{$-13.42$} & $[-13.65,\,-13.20]$ & $-148.1$ & $0.250$ \\
\midrule
\multicolumn{6}{l}{\textbf{Secondary: VGG-16, CIFAR-10, SGD (exponential LR)}} \\
\midrule
Acc.\ (\%) & Baseline & \lose{$-0.19$} & $[-0.48,\,0.10]$ & $-1.60$ & $0.250$ \\
Acc.\ (\%) & IES (shipped) & \lose{$-0.20$} & $[-0.57,\,0.17]$ & $-1.36$ & $0.250$ \\
Acc.\ (\%) & IES (literal) & \lose{$-0.04$} & $[-0.50,\,0.42]$ & $-0.22$ & $1.000$ \\
Backprop Saved (\%) & Baseline & \win{$54.40$} & $[54.11,\,54.69]$ & $470.5$ & $0.250$ \\
Backprop Saved (\%) & IES (shipped) & \win{$16.72$} & $[14.46,\,18.98]$ & $18.4$ & $0.250$ \\
Backprop Saved (\%) & IES (literal) & \win{$13.42$} & $[10.69,\,16.13]$ & $12.3$ & $0.250$ \\
FLOP Saved (\%) & Baseline & \win{$36.23$} & $[36.04,\,36.43]$ & $470.5$ & $0.250$ \\
FLOP Saved (\%) & IES (shipped) & \win{$11.14$} & $[9.63,\,12.64]$ & $18.4$ & $0.250$ \\
FLOP Saved (\%) & IES (literal) & \win{$8.94$} & $[7.12,\,10.75]$ & $12.3$ & $0.250$ \\
\midrule
\multicolumn{6}{l}{\textbf{Adam: ResNet-18, CIFAR-10}} \\
\midrule
Acc.\ (\%) & Baseline & \lose{$-0.46$} & $[-0.87,\,-0.06]$ & $-2.84$ & $0.250$ \\
Acc.\ (\%) & IES (shipped) & \lose{$-0.17$} & $[-0.58,\,0.24]$ & $-1.01$ & $0.250$ \\
Acc.\ (\%) & IES (literal) & \win{$0.02$} & $[-0.44,\,0.48]$ & $0.13$ & $1.000$ \\
Backprop Saved (\%) & Baseline & \win{$30.15$} & $[28.61,\,31.69]$ & $48.6$ & $0.250$ \\
Backprop Saved (\%) & IES (shipped) & \lose{$-25.85$} & $[-27.28,\,-24.41]$ & $-44.7$ & $0.250$ \\
Backprop Saved (\%) & IES (literal) & \lose{$-29.91$} & $[-32.18,\,-27.65]$ & $-32.8$ & $0.250$ \\
FLOP Saved (\%) & Baseline & \win{$20.09$} & $[19.06,\,21.11]$ & $48.6$ & $0.250$ \\
FLOP Saved (\%) & IES (shipped) & \lose{$-17.22$} & $[-18.18,\,-16.26]$ & $-44.7$ & $0.250$ \\
FLOP Saved (\%) & IES (literal) & \lose{$-19.93$} & $[-21.44,\,-18.42]$ & $-32.8$ & $0.250$ \\
\midrule
\multicolumn{6}{l}{\textbf{CIFAR-100: ResNet-18, SGD (exponential LR)}} \\
\midrule
Acc.\ (\%) & Baseline & \lose{$-0.87$} & $[-1.90,\,0.16]$ & $-2.09$ & $0.250$ \\
Acc.\ (\%) & IES (shipped) & \lose{$-0.86$} & $[-1.55,\,-0.17]$ & $-3.09$ & $0.250$ \\
Acc.\ (\%) & IES (literal) & \lose{$-0.77$} & $[-1.64,\,0.11]$ & $-2.18$ & $0.250$ \\
Backprop Saved (\%) & Baseline & \win{$57.73$} & $[57.64,\,57.83]$ & $1522.5$ & $0.250$ \\
Backprop Saved (\%) & IES (shipped) & \win{$45.79$} & $[45.15,\,46.43]$ & $178.2$ & $0.250$ \\
Backprop Saved (\%) & IES (literal) & \win{$39.27$} & $[38.16,\,40.37]$ & $88.3$ & $0.250$ \\
FLOP Saved (\%) & Baseline & \win{$38.47$} & $[38.40,\,38.53]$ & $1522.5$ & $0.250$ \\
FLOP Saved (\%) & IES (shipped) & \win{$30.51$} & $[30.08,\,30.94]$ & $178.2$ & $0.250$ \\
FLOP Saved (\%) & IES (literal) & \win{$26.17$} & $[25.43,\,26.90]$ & $88.3$ & $0.250$ \\
\midrule
\multicolumn{6}{l}{\textbf{DenseNet-121: CIFAR-10, SGD (exponential LR)}} \\
\midrule
Acc.\ (\%) & Baseline & \lose{$-0.08$} & $[-0.64,\,0.47]$ & $-0.37$ & $1.000$ \\
Acc.\ (\%) & IES (shipped) & \lose{$-0.08$} & $[-0.57,\,0.40]$ & $-0.43$ & $1.000$ \\
Acc.\ (\%) & IES (literal) & \win{$0.04$} & $[-0.21,\,0.28]$ & $0.37$ & $1.000$ \\
Backprop Saved (\%) & Baseline & \win{$54.49$} & $[52.85,\,56.13]$ & $82.5$ & $0.250$ \\
Backprop Saved (\%) & IES (shipped) & \win{$5.92$} & $[4.36,\,7.49]$ & $9.4$ & $0.250$ \\
Backprop Saved (\%) & IES (literal) & \win{$2.92$} & $[2.32,\,3.53]$ & $12.0$ & $0.250$ \\
FLOP Saved (\%) & Baseline & \win{$36.32$} & $[35.22,\,37.41]$ & $82.5$ & $0.250$ \\
FLOP Saved (\%) & IES (shipped) & \win{$3.95$} & $[2.90,\,4.99]$ & $9.4$ & $0.250$ \\
FLOP Saved (\%) & IES (literal) & \win{$1.95$} & $[1.54,\,2.35]$ & $12.0$ & $0.250$ \\
\midrule
\multicolumn{6}{l}{\textbf{ViT-Tiny: CIFAR-10, AdamW (warmup-cosine)}} \\
\midrule
Acc.\ (\%) & Baseline & \lose{$-0.36$} & $[-1.43,\,0.71]$ & $-0.84$ & $1.000$ \\
Acc.\ (\%) & IES (shipped) & \lose{$-0.21$} & $[-0.94,\,0.52]$ & $-0.71$ & $1.000$ \\
Acc.\ (\%) & IES (literal) & \lose{$-0.15$} & $[-0.77,\,0.47]$ & $-0.61$ & $1.000$ \\
Backprop Saved (\%) & Baseline & \win{$25.49$} & $[24.84,\,26.13]$ & $98.5$ & $0.250$ \\
Backprop Saved (\%) & IES (shipped) & \win{$12.38$} & $[12.25,\,12.52]$ & $228.0$ & $0.250$ \\
Backprop Saved (\%) & IES (literal) & \win{$8.83$} & $[8.62,\,9.05]$ & $101.1$ & $0.250$ \\
FLOP Saved (\%) & Baseline & \win{$17.06$} & $[16.63,\,17.49]$ & $98.5$ & $0.250$ \\
FLOP Saved (\%) & IES (shipped) & \win{$8.29$} & $[8.20,\,8.38]$ & $228.0$ & $0.250$ \\
FLOP Saved (\%) & IES (literal) & \win{$5.91$} & $[5.77,\,6.06]$ & $101.1$ & $0.250$ \\
\midrule
\multicolumn{6}{l}{\textbf{ViT-Tiny: CIFAR-100, AdamW (warmup-cosine)}} \\
\midrule
Acc.\ (\%) & Baseline & \lose{$-0.69$} & $[-2.36,\,0.98]$ & $-1.03$ & $0.250$ \\
Acc.\ (\%) & IES (shipped) & \lose{$-0.63$} & $[-2.33,\,1.07]$ & $-0.92$ & $1.000$ \\
Acc.\ (\%) & IES (literal) & \lose{$-0.21$} & $[-1.97,\,1.55]$ & $-0.30$ & $1.000$ \\
Backprop Saved (\%) & Baseline & \win{$26.25$} & $[25.90,\,26.59]$ & $191.0$ & $0.250$ \\
Backprop Saved (\%) & IES (shipped) & \win{$17.89$} & $[17.60,\,18.18]$ & $154.5$ & $0.250$ \\
Backprop Saved (\%) & IES (literal) & \win{$13.77$} & $[13.36,\,14.17]$ & $85.0$ & $0.250$ \\
FLOP Saved (\%) & Baseline & \win{$17.57$} & $[17.34,\,17.79]$ & $191.0$ & $0.250$ \\
FLOP Saved (\%) & IES (shipped) & \win{$11.98$} & $[11.78,\,12.17]$ & $154.5$ & $0.250$ \\
FLOP Saved (\%) & IES (literal) & \win{$9.21$} & $[8.94,\,9.48]$ & $85.0$ & $0.250$ \\
\end{longtable}
\endgroup

\textbf{What this does and does not support.} On the primary cell, TG-IES's backprop- and FLOP-saved advantage over every other method is in the same direction in all 3 seeds for every comparison (hence $p_{\mathrm{sign}}=0.25$, the floor), with CIs entirely on one side of zero; large $d_z$ values ($12$--$470$) reflect small seed-to-seed variability relative to a large, consistent savings effect. The same pattern holds on VGG-16, CIFAR-100 ($d_z=1522.5$, the largest in this table), DenseNet-121, and both ViT-Tiny cells. On ResNet-18(fixed-LR) and Adam, the sign flips as \cref{sec:results,app:adam} report qualitatively, with CIs again excluding zero entirely. Accuracy differences are small ($|\text{mean }\Delta|\le0.87$ points) with CIs mostly straddling zero; the exceptions (Adam vs.\ Baseline, CIFAR-100 vs.\ IES shipped) are real, if small, effects at this sample size. A CI containing zero is not evidence the true difference \emph{is} zero, and a large $d_z$ with $p_{\mathrm{sign}}=0.25$ is not conventionally ``significant'' -- it is the most a 3-seed exact sign test can report regardless of effect size; we do not describe any result here as statistically significant at a threshold this design cannot reach.

\end{document}

%% file: fig_overview_tikz.tex

\definecolor{tgdark}{HTML}{2F5D8C}
\definecolor{tgfillA}{HTML}{EAF2FB}
\definecolor{tgfillB}{HTML}{D6E6F5}
\begin{tikzpicture}[
  font=\scriptsize,
  box/.style={draw=tgdark, line width=0.7pt, rounded corners=2pt, fill=tgfillA,
              align=center, text width=2.95cm, minimum height=1.15cm, inner sep=2.5pt},
  theory/.style={box, fill=tgfillB},
  analysis/.style={box, dashed, fill=white, text width=3.4cm},
  arr/.style={-{Stealth[length=2mm]}, line width=0.8pt, draw=tgdark},
  darr/.style={arr, dashed},
]
\node[box] (A) at (0,0)    {\textbf{Noisy loss} $\tilde L_i^{(t)}$\\of \emph{every} instance\\(forward pass)};
\node[box] (B) at (3.55,0) {\textbf{Second-order signal}\\$\hat S_i=|\Delta^2\tilde L_i^{(t)}|$};
\node[box] (C) at (7.1,0)  {\textbf{Mastery check}\\$\hat S_i<\delta$ for $K^*$ epochs\\$\Rightarrow$ passive set $P_t$};
\node[box] (D) at (10.65,0){\textbf{Selective update}\\backward on $A_t$ only:\\$\hat g^{(t)}=\frac1N\sum_{i\in A_t}g_i$};
\draw[arr] (A) -- (B);
\draw[arr] (B) -- (C);
\draw[arr] (C) -- (D);
\draw[arr, draw=gray] (D.north) -- ++(0,0.4) -| (A.north)
  node[pos=0.25, above, font=\scriptsize, text=black]
  {next epoch: all losses are re-observed; a passive instance returns once $\hat S_i\ge\delta$};
\node[theory]   (E) at (3.55,-2.05) {\textbf{Threshold from noise}\\$\delta=k\sqrt6\,\hat\sigma$};
\node[theory]   (F) at (6.75,-2.05) {\textbf{Patience} $K^*$\\initialized from $K_{\mathrm{safe}}$};
\node[analysis] (G) at (10.45,-2.05){\textbf{Analysis of the update}\\$\hat g^{(t)}=g-\omega_t$ exactly; $\|\omega_t\|$ bounds drift \& descent};
\draw[arr] (E.north) .. controls +(0,0.5) and +(0,-0.5) .. ([xshift=-0.9cm]C.south);
\draw[arr] (F.north) -- (F.north |- C.south);
\draw[darr] ([xshift=1.15cm]C.south) -- ([xshift=-1.45cm]G.north)
  node[midway, right=1pt, font=\scriptsize] {};
\draw[darr] ([xshift=1.0cm]G.north) -- ([xshift=1.0cm]G.north |- D.south);
\end{tikzpicture}